\documentclass[twoside]{article}

\usepackage[accepted]{aistats2020}
\usepackage[round]{natbib}

\usepackage{color}
\usepackage{Definitions}
\usepackage{algorithm}
\usepackage[noend]{algorithmic}
\usepackage{subcaption}
\usepackage{graphicx}
\usepackage{mathtools}
\usepackage{booktabs}

\DeclarePairedDelimiter\floor{\lfloor}{\rfloor}

\newcommand\blfootnote[1]{%
  \begingroup
  \renewcommand\thefootnote{}\footnote{#1}%
  \addtocounter{footnote}{-1}%
  \endgroup
}

\usepackage{amsmath}

\begin{document}

%

%

\twocolumn[

\aistatstitle{Ordered SGD: A New Stochastic Optimization Framework for Empirical Risk Minimization}

\aistatsauthor{Kenji Kawaguchi\textsuperscript{*} \And Haihao Lu\textsuperscript{*}}

\aistatsaddress{MIT  \And Google Research} 

]

\begin{abstract}
We propose a new stochastic optimization framework for empirical risk minimization problems such as those that arise in machine learning. The traditional approaches, such as (mini-batch) stochastic gradient descent (SGD), utilize an unbiased gradient estimator of the empirical average loss. In contrast, we develop a computationally efficient method to construct a gradient estimator that is purposely biased toward those observations with higher current losses. On the theory side,  we show that the proposed method minimizes a new ordered modification of the empirical average loss, and is guaranteed to converge at a sublinear rate to a global optimum for convex loss and to a critical point for weakly convex (non-convex) loss.  Furthermore, we prove a new generalization bound for the proposed algorithm. On the empirical side, the numerical experiments show that our proposed method consistently improves the test errors compared with the standard mini-batch SGD in various models including SVM, logistic regression, and deep learning problems.  
\end{abstract}

\section{Introduction} \label{sec:intro}
Stochastic Gradient Descent (SGD), as the workhorse training algorithm for most machine learning applications including deep learning, has been extensively studied in recent years (e.g., see a recent review  by \citealt{bottou2018optimization}). At every step, SGD draws one training sample uniformly at random from the training dataset, and then uses the (sub-)gradient of the loss over the selected sample to update the model parameters. The most popular version of SGD in practice is perhaps the mini-batch  SGD \citep{bottou2018optimization,dean2012large}, which is widely implemented in the state-of-the-art deep learning frameworks, such as TensorFlow \citep{abadi2016tensorflow}, PyTorch \citep{paszke2017automatic} and CNTK \citep{seide2016cntk}. Instead of choosing one sample per iteration,  mini-batch SGD randomly selects a mini-batch of the samples, and uses the (sub-)gradient of the average loss over the  selected samples to update the model parameters. 

Both SGD and mini-batch SGD utilize uniform sampling during the entire learning process, so that the stochastic gradient is always an unbiased gradient estimator of the empirical average loss over all samples. On the other hand, it appears to practitioners that not all samples are equally important, and indeed most of them could be ignored after a few epochs of training without affecting the final model
\citep{katharopoulos2018not}. For example, intuitively, the samples near the final decision boundary should be more important to build the model than those far away from the boundary for classification problems. In particular, as we will illustrate later in Figure \ref{fig:2d_illustration}, there are cases when those far-away samples may corrupt the model by using average loss. In order to further explore such structures, we propose an efficient sampling scheme on top of the mini-batch SGD. We call the resulting algorithm   \textit{ordered SGD}, which is used to learn a different type of models with the goal to improve the testing performance. \blfootnote{*equal contribution}

The above motivation of ordered SGD is  related to that of importance sampling SGD, which has been extensively studied recently in order to improve the convergence speed of SGD \citep{needell2014stochastic,zhao2015stochastic,alain2015variance,loshchilov2015online,gopal2016adaptive,katharopoulos2018not}. However, our goals, algorithms and theoretical results are fundamentally different from those in the previous studies on importance sampling SGD. Indeed, all aforementioned studies are aimed to  accelerate the minimization process for the  empirical average loss, whereas our proposed method turns out to minimize a new objective function by purposely constructing a biased gradient. 

Our main contributions can be summarized as follows: i) we propose a computationally efficient and easily implementable algorithm, ordered SGD, with principled  motivations (Section \ref{sec:algo}), ii) we show that  ordered SGD minimizes an  ordered empirical risk  with sub-linear rate for convex and weakly convex (non-convex) loss functions  (Section \ref{sec:optimization}), iii) we prove a generalization bound for ordered SGD (Section \ref{sec:generalization}), and iv) our numerical experiments show ordered SGD consistently improved mini-batch SGD in test errors (Section \ref{sec:nume}).

\section{Empirical Risk Minimization}

Empirical risk minimization is one of the main tools to build a model in machine learning. Let  $\Dcal=((x_i,y_i))_{i=1}^n$ be a training dataset of  $n$ samples where $x_i \in\Xcal \subseteq  \RR^{d_x}$ is the input vector and $y_i \in\Ycal \subseteq  \RR^{d_y}$ is the target output vector for the $i$-th sample. The goal of empirical risk minimization is to find a prediction function $f(\hspace{1pt}\cdot\hspace{2pt};\theta): \RR^{d_x}  \rightarrow \RR^{d_y}$, by minimizing 
\begin{equation} \label{eq:poi}
L(\theta ) := \frac{1}{n}\sum_{i=1}^n  L_i(\theta)+R(\theta),
\end{equation}
where $\theta\in\RR^{d_\theta}$ is the parameter vector of the prediction model, $L_i(\theta) := \ell(f(x_{i};\theta),y_{i})$ with the function $\ell: \RR^{d_y} \times\Ycal \rightarrow \RR_{\ge 0}$ is the loss of the $i$-th sample, and $R(\theta)\ge0$ is a regularizer. For example, in logistic regression, $f(x;\theta)=\theta^T x$ is a linear function of the input vector $x$, and $\ell(a,y)=\log(1+\exp(-y a))$ is the logistic loss function with $y\in\{-1,1\}$. For a  neural network, $f(x;\theta)$ represents the pre-activation output of the last layer.


\section{Algorithm} \label{sec:algo}

In this section, we  introduce  ordered SGD and  provide an intuitive explanation of the advantage of ordered SGD by looking at 2-dimension toy  examples with linear classifiers and  small artificial neural networks (ANNs). Let us first introduce a new notation $\qargmax$ as an extension to the standard notation $\argmax$:
\begin{definition}
Given a set of $n$ real numbers $(a_{1},a_2,\dots,a_n)$, an index subset $S\subseteq \{1,2,\ldots,n\}$, and a positive integer number $q\le |S|$, we define $\qargmax_{j\in S} \hspace{2pt} a_j$ such that $Q \in \qargmax_{j\in S} \hspace{2pt} a_j$ is a set of $q$ indexes of the $q$ largest values of $(a_j)_{j \in S}$; i.e.,  $\qargmax_{j\in S} \hspace{2pt} a_j = \argmax_{Q \subseteq S, |Q|=q} \sum_{i\in Q} a_i$.
\end{definition}

\begin{algorithm}[t!]
\caption{Ordered Stochastic Gradient Descent (ordered SGD)} \label{al:qSGD} 
$ \ $
\begin{algorithmic}[1] 
\STATE {\bf Inputs:}  an initial vector  $\theta^0$ and a learning rate sequence $(\eta_k)_{k}$  
\FOR {$t = 1,2,\ldots$} 
  \STATE Randomly choose a  mini-batch of samples: $S\subseteq \{1, 2, \ldots, n\}$ such that $|S|=s$.
  \STATE Find a set $Q$ of top-$q$ samples in $S$ in term of loss values: $Q \in \qargmax_{i\in S} L_i(\theta^t)$.
  \STATE Compute a subgradient $\tg^{t}$ of the top-$q$ samples $L_{Q}(\theta^t)$:  $\tg^{t} \in \partial L_{Q}(\theta^t)$ where $ L_{Q}(\theta^t)= \frac{1}{q} \sum_{i\in Q}  L_i(\theta^t) +  R(\theta^{t})$ and $\partial L_Q$ is the set of sub-gradient\footnotemark of function $L_Q$.
  \STATE Update parameters $\theta$: $\theta^{t+1}=\theta^t - \eta_t  \tg^{t}$
\ENDFOR   
\end{algorithmic}
\end{algorithm}

Algorithm \ref{al:qSGD} describes the pseudocode of our proposed algorithm, ordered SGD. The procedures of ordered SGD follow those of  mini-batch SGD except the following modification:  after drawing a mini-batch of size $s$, ordered SGD updates the parameter vector $\theta$ based on the (sub-)gradient of the average loss over the top-$q$ samples in the mini-batch in terms of individual loss values (lines 4 and 5 of Algorithm \ref{al:qSGD}). This modification  is   used to purposely build and utilize a biased gradient estimator with more weights on the samples having larger losses. As it can be seen in Algorithm \ref{al:qSGD}, ordered SGD is  easily implementable, requiring to change only a single line or few lines on top of a mini-batch SGD implementation. 

\footnotetext{The sub-gradient for (non-convex) $\rho$-weakly convex function $L_Q$ at $\theta^t$ is defined as $\{g|L_Q(\theta)\ge L_Q(\theta^t) + \langle g,\theta-\theta^t\rangle - \frac{\rho}{2}\|\theta-\theta^t\|^2, \forall{\theta}\}$ \citep{rockafellar2009variational}.}

\begin{figure*}[t!]
\centering
\begin{subfigure}[b]{0.245\textwidth}\centering
  \includegraphics[width=\textwidth,height=0.6\textwidth]{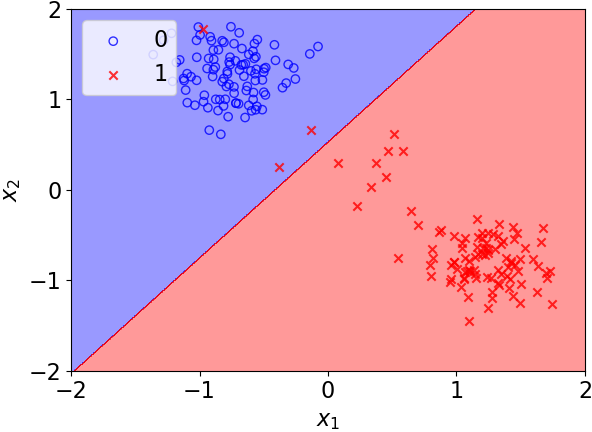}
\end{subfigure} 
\begin{subfigure}[b]{0.245\textwidth}\centering
  \includegraphics[width=\textwidth,height=0.6\textwidth]{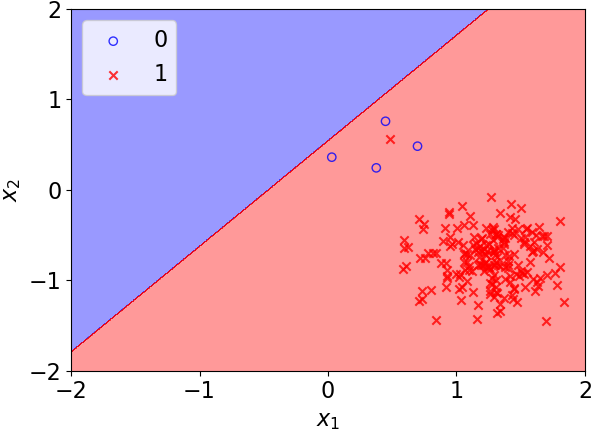}
\end{subfigure} 
\begin{subfigure}[b]{0.245\textwidth}\centering
  \includegraphics[width=\textwidth,height=0.6\textwidth]{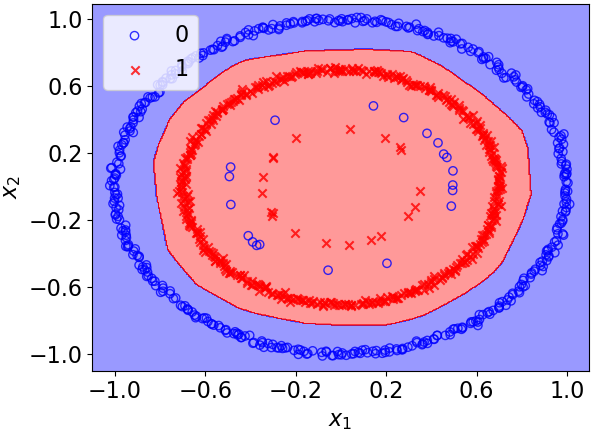}
\end{subfigure}
\begin{subfigure}[b]{0.245\textwidth}
  \includegraphics[width=\textwidth,height=0.6\textwidth]{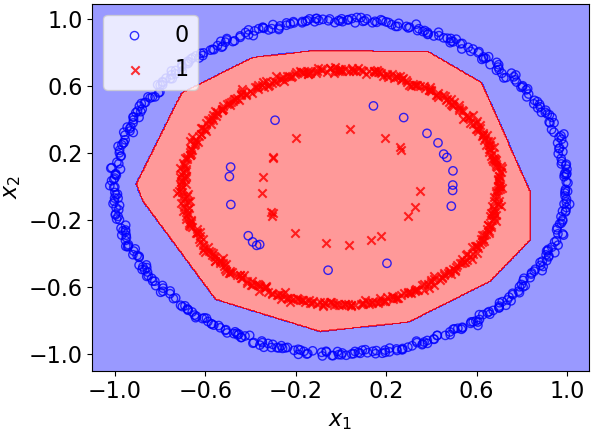}
\end{subfigure}
\begin{subfigure}[b]{0.245\textwidth}\centering
  \includegraphics[width=\textwidth,height=0.6\textwidth]{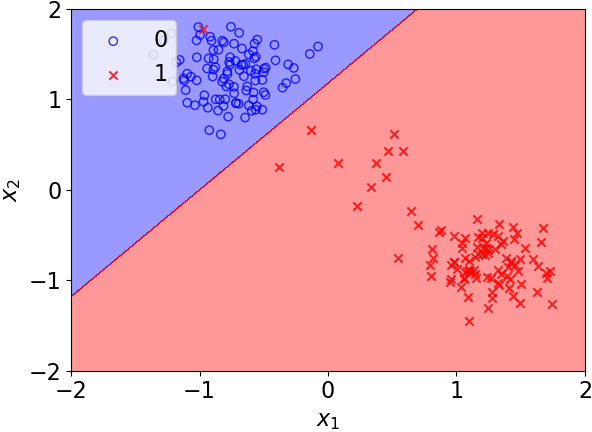}
\caption{with linear classifier} \label{fig:2d_illustration_a}
\end{subfigure} 
\begin{subfigure}[b]{0.245\textwidth}\centering
  \includegraphics[width=\textwidth,height=0.6\textwidth]{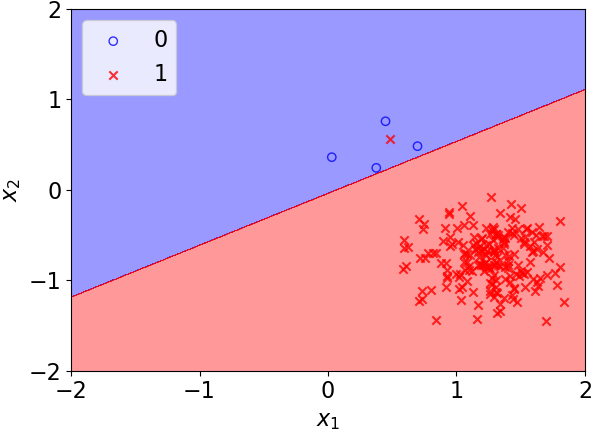}
  \caption{with linear classifier} \label{fig:2d_illustration_b}
\end{subfigure} 
\begin{subfigure}[b]{0.245\textwidth}\centering
  \includegraphics[width=\textwidth,height=0.6\textwidth]{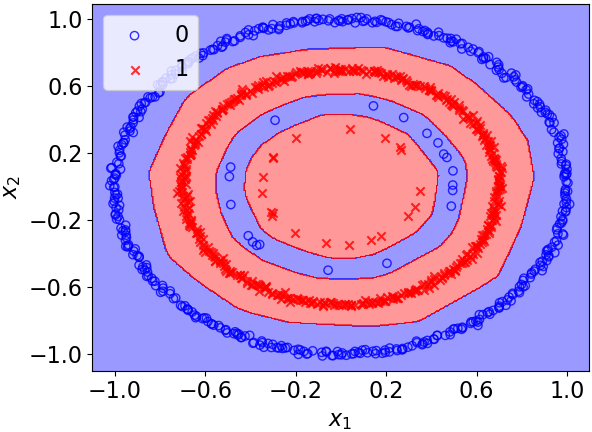}
  \caption{with small ANN} \label{fig:2d_illustration_c}
\end{subfigure}
\begin{subfigure}[b]{0.245\textwidth}
  \includegraphics[width=\textwidth,height=0.6\textwidth]{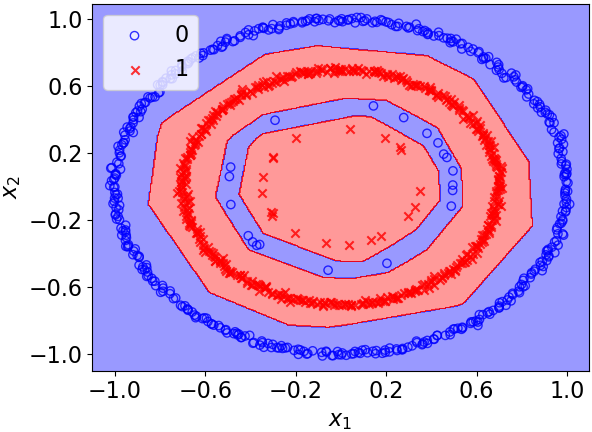}
  \caption{with tiny ANN} \label{fig:2d_illustration_d}
\end{subfigure}
\caption{Decision boundaries of mini-batch SGD predictors (\textbf{top} row) and ordered SGD predictors (\textbf{bottom} row) with 2D synthetic datasets for binary classification. In these  examples, ordered SGD predictors correctly classify more data points than mini-batch SGD predictors, because a ordered SGD predictor can focus more on a  smaller yet informative subset of data points, instead of focusing on the average loss dominated by  a larger subset of data points. } 
\label{fig:2d_illustration}
\end{figure*}

Figure \ref{fig:2d_illustration} illustrates the motivation of ordered SGD by looking at two-dimensional toy problems of binary classification. To avoid an extra freedom due to the hyper-parameter $q$, we employed a  single fixed procedure to  set the hyper-parameter $q$ in the experiments for Figure \ref{fig:2d_illustration} and other experiments in Section \ref{sec:nume}, which is further explained in  Section \ref{sec:nume}. The details of the experimental settings for Figure \ref{fig:2d_illustration} are presented in Section \ref{sec:nume} and in Appendix \ref{sec:app:exp}. 

It can be seen from Figure \ref{fig:2d_illustration} that ordered SGD adapts better to imbalanced data distributions compared with mini-batch SGD. It can better capture the information of the smaller sub-clusters that contribute less to the empirical average loss $L(\theta)$: e.g.,  the  small sub-clusters in the middle of Figures \ref{fig:2d_illustration_a} and  \ref{fig:2d_illustration_b}, as well as the small inner ring structure in Figures \ref{fig:2d_illustration_c} and \ref{fig:2d_illustration_d} (the two inner rings contain only 40 data points while the two outer rings contain 960 data points). The smaller sub-clusters  are informative for training a classifier when  they are not outliers or by-products of noise.  A sub-cluster of data points would be less  likely to be an outlier as the size of the sub-cluster increases. The value of $q$ in  ordered SGD can control the size of sub-clusters that a classifier should be sensitive to. With smaller $q$, the output model becomes more sensitive to smaller sub-clusters. In an extreme case with $q=1$ and $n=s$, ordered SGD minimizes the maximal loss \citep{shalev2016minimizing} that is highly sensitive to every smallest sub-cluster of each single data point.


\vspace{-5pt}
\section{Optimization Theory} \label{sec:optimization}
\vspace{-5pt}

In this section, we answer the following three questions: (1) what objective function does ordered SGD solve as an optimization method, (2) what is the convergence rate of ordered SGD for minimizing the new objective function, and (3) what is the asymptotic structure of the new objective function.

Similarly to the notation of order statistics, we first introduce the notation of ordered indexes: given a model parameter $\theta$, let $ L_{(1)}(\theta) \ge L_{(2)}(\theta) \ge \cdots \ge L_{(n)}(\theta)$ be the decreasing values of the individual losses   $L_1(\theta),\ldots,L_n(\theta)$, where $(j)\in\{1,\dots,n\}$  (for all $j \in \{1,\dots,n\}$). That is, $\{(1),\ldots,(n)\}$ as a perturbation of $\{1,\dots,n\}$ defines the order of sample indexes by loss values. Throughout this paper,  whenever we  encounter ties on the values, we employ a tie-breaking rule in order to ensure the uniqueness of such an order.\footnote{In the case of ties, the order is defined by the order of the original indexes $(1,2,\ldots,n)$ of $L_1(\theta),\dots,L_n(\theta)$; i.e., if $L_{i_1}(\theta)=L_{i_2}(\theta)$ and $i_1<i_2$, then $i_1$ appears before $i_2$ in the sequence $({(1)}, {(2)}, \dots, {(n)})$. }  Theorem \ref{thm:equiv} shows that ordered SGD is a stochastic first-order method for minimizing the new ordered empirical loss $L_q(\theta)$.
\begin{theorem}\label{thm:equiv}
Consider the following objective function: \vspace{-8pt}
\begin{equation}{\label{eq:Lq}}
    L_{q}(\theta) := \frac{1}{q} \sum_{j=1}^n \gamma_{j} L_{(j)}(\theta)+R(\theta),
\end{equation} 
where the parameter $\gamma_j$ depends on the tuple $(n,s,q)$, and is defined by \vspace{-8pt}
\begin{equation}\label{eq:gamma}
    \gamma_{j}:=\frac{\sum_{l=0}^{q-1}\binom{j-1}{l}\binom{n-j}{s-l-1}}{\binom{n}{s}}.
\end{equation}
Then, ordered SGD is a stochastic first-order method for minimizing $L_q(\theta)$ in the sense that $\tg^{t}$ used in ordered SGD is an unbiased estimator of a (sub-)gradient of $L_q(\theta)$.
\end{theorem}

\begin{figure*}[t!]
\centering
\begin{subfigure}[b]{0.329\textwidth}\centering
  \includegraphics[width=\textwidth,height=0.7\textwidth]{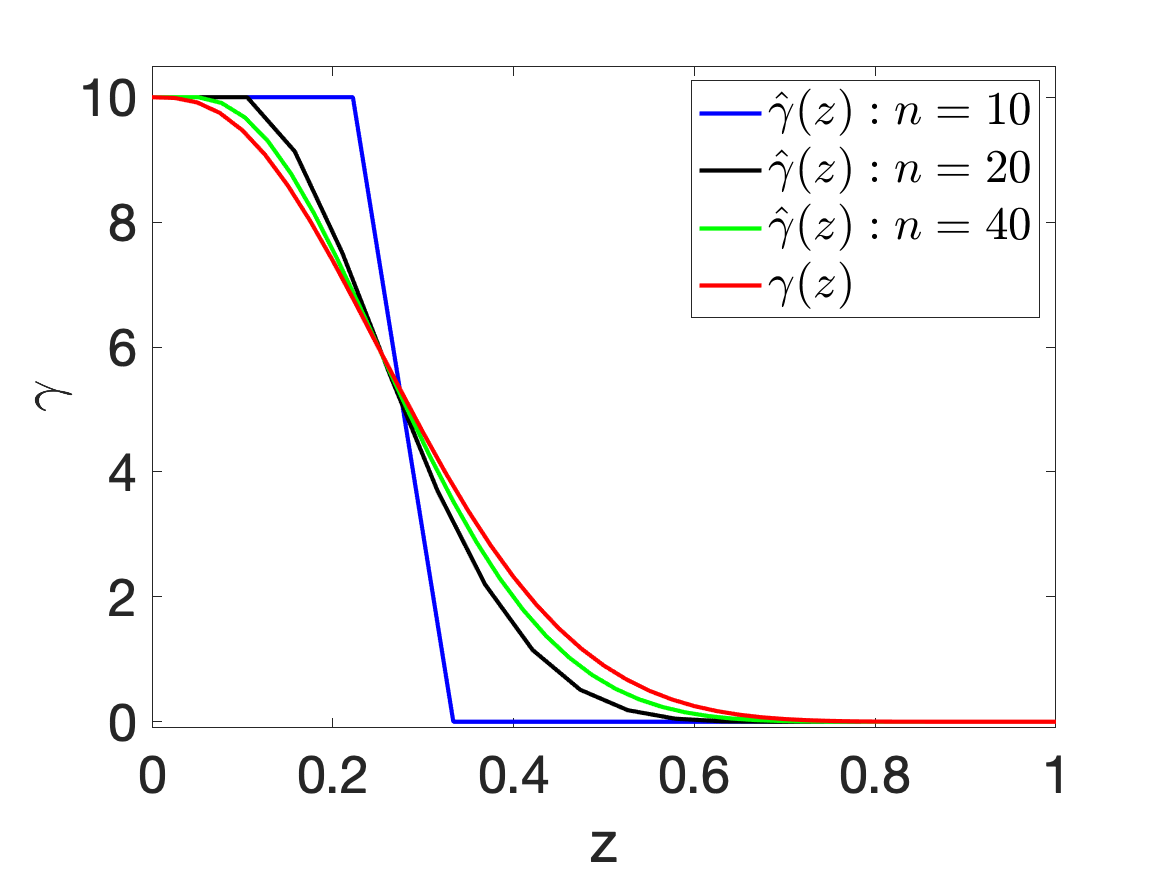}
\caption{$(s, q)=(10, 3)$} \label{fig:sub_gamma_1}
\end{subfigure} 
\begin{subfigure}[b]{0.329\textwidth}\centering
  \includegraphics[width=\textwidth,height=0.7\textwidth]{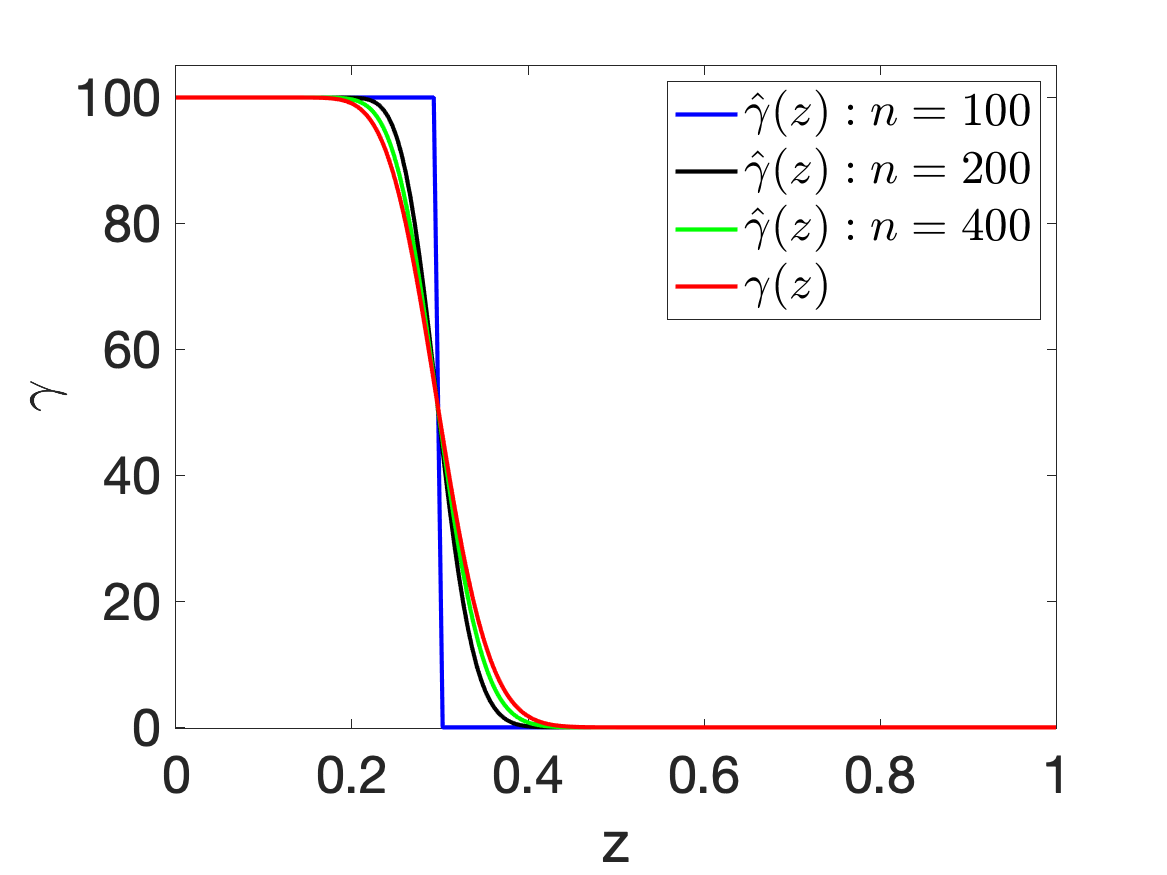}
  \caption{$(s, q)=(100, 30)$} \label{fig:sub_gamma_2}
\end{subfigure} 
\begin{subfigure}[b]{0.329\textwidth}\centering
  \includegraphics[width=\textwidth,height=0.7\textwidth]{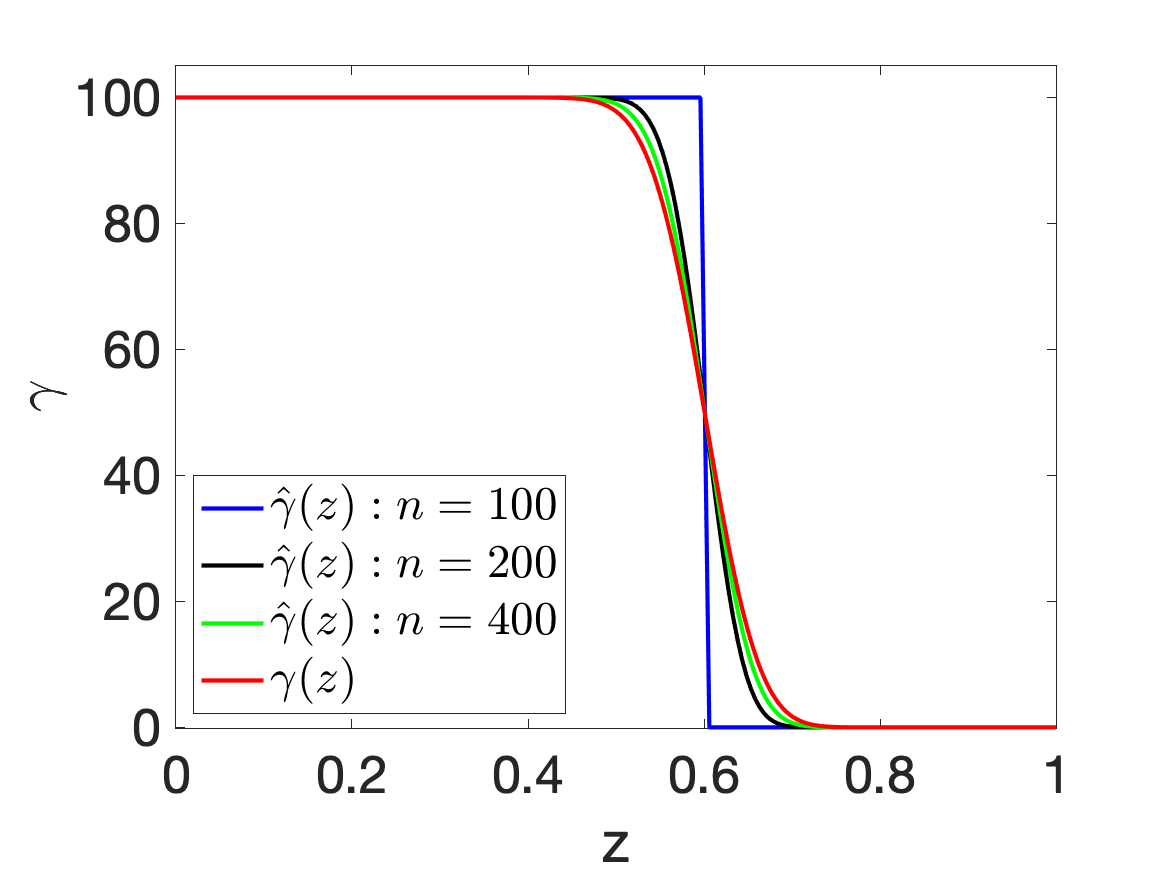}
  \caption{$(s, q)=(100, 60)$} \label{fig:sub_gamma_3}
\end{subfigure} 
\caption{$\hgamma(z)$ and $\gamma(z)$ for different $(n,s,q)$ where   $\hat{\gamma}$ is a rescaled version of $\gamma_j$: $\hgamma(j/n)=n\gamma_j$ . }
\label{fig:plots_gamma}
\end{figure*}

Although the order of individual losses change with different $\theta$, $L_q$ is a well-defined function. For any given $\theta$, the order of individual losses is fixed and $L_q(\theta)$ has a unique value, which means
$L_q(\theta)$ is a function of $\theta$.

All proofs in this paper are deferred to Appendix \ref{sec:app:proof}. As we can see from Theorem \ref{thm:equiv},  the objective function minimized by ordered SGD (i.e., $L_{q}(\theta)$) depends on the hyper-parameters of the algorithm through the values of $\gamma_{j}$. Therefore, it is of practical interest to obtain deeper understandings on how the hyper-parameters $(n,s,q)$ affects the objective function $L_q(\theta)$ through $\gamma_{j}$. The next proposition presents the asymptotic value of $\gamma_j$ (when $n\rightarrow\infty$), which shows that a  rescaled $\gamma_j$ converges to the cumulative distribution function of a Beta distribution:
\begin{proposition}\label{prop:limit}
Denote $z=\frac{j}{n}$ and  $\gamma(z):=\sum_{l=0}^{q-1}z^{l}(1-z)^{s-l-1}\frac{s!}{l!(s-l-1)!}$. Then, it holds that 
$$
\lim_{j,n\rightarrow \infty, j/n=z} \gamma_j = \frac{1}{n} \gamma(z).
$$
Moreover, it holds that $1-\frac{1}{s}\gamma(z)$ is the cumulative distribution function of $\text{Beta}(z;q, s-q)$. 
\end{proposition}

To better illustrate the structure of  $\gamma_j$ in the non-asymptotic regime,  Figure \ref{fig:plots_gamma} plots $\hat{\gamma}(z)$ and $\gamma(z)$ for different values of $(n,s,q)$ where  $\hat{\gamma}(z)$ is a rescaled version of $\gamma_j$ defined by $\hgamma{(j/n)}=n\gamma_j$ (and the value of $\hgamma(\cdot)$ between $j/n$ and $(j+1)/n$ is defined by linear interpolation for better visualization). As we can see from Figure \ref{fig:plots_gamma}, $\hat{\gamma}(z)$ monotonically decays.  In each subfigure, with fixed $s,q$, the cliff gets smoother and $\hat{\gamma}(z)$ converges to $\gamma(z)$ as $n$ increases. Comparing Figures \ref{fig:sub_gamma_1} and  \ref{fig:sub_gamma_2}, we can see that as $s$, $q$ and $n$ all increase proportionally, the cliff gets steeper.  Comparing Figures \ref{fig:sub_gamma_2} and  \ref{fig:sub_gamma_3}, we can see that with fixed $n$ and $q$, the cliff shifts to the right as $q$ increases.

As a direct extension of Theorem \ref{thm:equiv}, we can now obtain the computational guarantees of ordered SGD for minimizing $L_q(\theta)$ by taking advantage of the classic convergence results of SGD:

\begin{theorem}\label{cor:opt}
 Let  $(\theta^t)_{t=0}^{T}$ be a sequence generated by  ordered SGD (Algorithm \ref{al:qSGD}). Suppose that $L_i(\cdot)$ is $G_1$-Lipschitz continuous for $i=1,\ldots,n$, and  $R(\cdot)$ is $G_2$-Lipschitz continuous. Suppose that there exists a finite $\theta^* \in \argmin_\theta L_q(\theta)$ and $L_q(\theta^*)$ is finite. Then, the following two statements hold:
\vspace{-6pt}
\begin{enumerate}[leftmargin=20pt]
\item[(1)]
(Convex setting). If   $L_i(\cdot)$ and   $R(\cdot)$ are both convex,  for any step-size $\eta_t$, it holds that
\begin{align*}
&\min_{0\le t\le n} \EE [L_q(\theta^t) - L_q(\theta^*)] 
\\ &\le \frac{2(G_1^2+G_2^2)\sum_{t=0}^T \eta_t^2 + \|\theta^*-\theta^0\|^2}{2\sum_{t=0}^T \eta_t} \ .
\end{align*}
\item[(2)] 
(Weakly convex setting) Suppose that  $L_i(\cdot)$ is $\rho$-weakly convex (i.e., $L_i(\theta)+\frac{\rho}{2}\|\theta\|^2$ is convex) and $R(\cdot)$ is convex. Recall the definition of Moreau envelope: $L_q^{\lambda}(\theta):=\min_{\beta}\{L_q(\beta)+\frac{1}{2\lambda}\|\beta-\theta\|^2\}$. Denote $\btheta^T$ as a random variable taking value in $\{\theta^0,\theta^1,\ldots,\theta^T\}$ according to the probability distribution $\PP(\btheta^T=\theta^t)=\frac{\eta_t}{\sum_{t=0}^T \eta_t}$. Then for any constant $\hrho>\rho$, it holds that
\begin{align*}
&\EE[\|\nabla L_q^{1/\hrho}(\btheta^T)\|^2] 
\\ & \scalebox{0.9}{$\displaystyle \le \frac{\hrho}{\hrho-\rho}\frac{\left(L_q^{1/\hrho}(\theta^0) - L_q(\theta^*)\right)+\hrho (G_1^2+G_2^2)\sum_{t=0}^T\eta_t^2}{\sum_{t=0}^T\eta_t} .$}
\end{align*}

\end{enumerate}
\end{theorem}


Theorem \ref{cor:opt} shows that in particular, if we choose $\eta_t \sim O(1/\sqrt{t})$,  the optimality gap $\min_t L_q(\theta^t) - L_q(\theta^*)$ and $\EE[\|\nabla L_q^{1/\hrho}(\btheta_T)\|^2]$ decay at the rate of $\tilde O(1/\sqrt{t})$  (note that $\lim_{T \rightarrow \infty} \frac{\sum_{t=0}^T \eta_t^2}{\sum_{t=0}^T \eta_t}=0$ with $\eta_t \sim O(1/\sqrt{t})$). 

The Lipschitz continuity assumption  in Theorem \ref{cor:opt} is a standard assumption for the analysis of stochastic optimization algorithms. This assumption is generally  satisfied with logistic loss, hinge loss and Huber loss without any constraints on $\theta^t$, and with square loss when one can presume that $\theta^t$ stays in a  compact space (which is typically the case being interested in practice). For the weakly convex setting, $\EE\|\nabla \varphi^{1/2\rho}(\theta^k)\|^2$ (appeared in Theorem \ref{cor:opt} (2)) is a natural measure of the near-stationarity for a non-differentiable weakly convex function $\varphi:\theta\mapsto\varphi(\theta)$ \citep{davis2018stochastic}. The weak convexity (also known as negative strong convexity or almost convexity) is a standard assumption for analyzing non-convex optimization problem in optimization literature \citep{davis2018stochastic,allen2017natasha}. With a standard loss criterion such as logistic loss, the individual objective $L_i(\cdot)$ with a neural network using sigmoid or tanh activation functions is weakly convex (neural network with ReLU activation function is not weakly convex and falls out of our setting). 



\vspace{-0.19cm}

\section{Generalization Bound} \label{sec:generalization}

\vspace{-0.19cm}

This section presents the generalization theory for ordered SGD. To make the dependence on a training dataset $\Dcal$ explicit, we  define $L(\theta;\Dcal):=\frac{1}{n}\sum_{i=1}^n  L_i(\theta;\Dcal)$ and $L_{q}(\theta;\Dcal):=\frac{1}{q} \sum_{j=1}^m \gamma_{j} L_{(j)}(\theta;\Dcal)$ by rewriting $L_i(\theta;\Dcal)=L_i(\theta)$ and $L_{(j)}(\theta;\Dcal)=L_{(j)}(\theta)$, where $((j))_{j=1}^n$ defines the order of sample indexes by the loss value, as stated in Section \ref{sec:optimization}. Denote $r_i(\theta; \Dcal) =  \sum_{j=1}^{n} \mathbbm{1}\{i=(j)\}\gamma_{j}$ where $(j)$ depends on $(\theta, \Dcal)$. Given an arbitrary set $\Theta\subseteq\RR^{d_\theta}$, we define  $\Rfra_{n}(\Theta)$ as the (standard) Rademacher
complexity of the set $\{ (x,y)\mapsto\ell(f(x;\theta),y): \theta \in \Theta\}$: 
$$
\Rfra_{n}(\Theta) =\EE_{\bar \Dcal,\xi} \left[\sup_{\theta \in \Theta} \frac{1}{n}\sum_{i=1}^n  \xi_{i}\ell(f(\bar x_{i};\theta), \bar y_{i}) \right],
$$
where $\overline\Dcal=((\bar x_i, \bar y_i))_{i=1}^n$, and $\xi_{1},\dots,\xi_{n}$ are independent uniform random variables taking values in $\{-1,1\}$ (i.e., Rademacher variables). Given a tuple $(\ell, f, \Theta,\Xcal,\Ycal)$, define $M$ as the least upper bound on the difference of individual loss values:      $|\ell(f(x;\theta),y)-\ell(f(x';\theta),y')| \le M $ for all $\theta \in\Theta $ and all  $(x,y),(x',y') \in \Xcal \times \Ycal$. For example, $M=1$ if $\ell$ is the  0-1 loss function.       
Theorem \ref{thm:ge_1} presents a generalization bound for ordered SGD:

\begin{theorem} \label{thm:ge_1}
Let $\Theta$ be a fixed subset of $\RR^{d_\theta}$. Then, for any $\delta>0$, with probability at least $1-\delta$ over an iid draw of $n$  examples $\mathcal{D}=((x_i,y_i))_{i=1}^n$, the following holds for all $\theta \in \Theta$:
\begin{align} \label{eq:ge_1}
&\EE_{(x,y)}[\ell(f(x;\theta),y)]
\\ \nonumber & \le L_{q}(\theta;\Dcal) + 2\Rfra_{n}(\Theta) + \frac{Ms}{q} \sqrt{\frac{\ln(1/\delta)}{2n}} - \Qcal_{n}(\Theta;s,q),
\end{align}
where 
$
\Qcal_{n}(\Theta;s,q) := \EE_{\bar \Dcal} [\inf_{\theta \in \Theta} \sum_{i=1}^n (\frac{r_i(\theta;\bar \Dcal)}{q}-\frac{1}{n})\ell(f(\bar x_{i};\theta), \bar y_{i}) ] \ge 0.
$
\end{theorem}

The expected error $\EE_{(x,y)}[\ell(f(x;\theta),y)]$ in the left-hand side of Equation \eqref{eq:ge_1} is a standard objective for generalization, whereas the right-hand side is an upper bound with the dependence on the algorithm parameters $q$ and $s$. Let us first look at the asymptotic case when  $n\rightarrow\infty$. Let $\Theta$ be constrained such that $\Rfra_n(\Theta)\rightarrow 0 $ as $n\rightarrow \infty$, which has been shown to be satisfied for various models and sets $\Theta$ \citep{bartlett2002rademacher,mohri2012foundations,bartlett2017spectrally,kawaguchi2017generalization}. With  $s/q$ being bounded, the third term in the right-hand side of Equation \eqref{eq:ge_1} disappear as $n\rightarrow\infty$. Thus, it holds with high probability that  $\EE_{(x,y)}[\ell(f(x;\theta),y)] \le L_{q}(\theta;\Dcal)-\Qcal_{n}(\Theta;s,q)\le L_{q}(\theta;\Dcal) $, where $L_{q}(\theta;\Dcal)$ is minimized by ordered SGD as shown in Theorem \ref{thm:equiv} and Theorem \ref{cor:opt}. From this viewpoint, ordered SGD minimizes the expected error for generalization when $n\rightarrow\infty$.

A special case of Theorem \ref{thm:ge_1} recovers the standard generalization bound of the empirical average loss (e.g., \citealp{mohri2012foundations}), That is, if $q=s$, ordered SGD becomes the standard mini-batch SGD and Equation \eqref{eq:ge_1}  becomes
\begin{align} \label{eq:ge_standard}
\EE_{(x,y)}[\ell(f(x;\theta),y)] \le L(\theta;\Dcal) + 2\Rfra_{n}(\Theta) + M \sqrt{\frac{\ln\frac{1}{\delta}}{2n}},
\end{align} 
which is the standard generalization bound  (e.g., \citealp{mohri2012foundations}). This is because if $q=s$, then  $\frac{r_i(\theta;\bar \Dcal)}{q}=\frac{1}{n}$ and hence $\Qcal_{n}(\Theta;s,q)=0$.  

For the purpose of a simple comparison of ordered SGD and (mini-batch) SGD, consider the case where we fix a single subset $\Theta \subseteq \RR^{d_\theta}$. Let $\hat \theta_{q}$ and $\hat \theta_s$ be the parameter vectors obtained by ordered SGD and (mini-batch) SGD respectively as the results of training.  Then, when  $n\rightarrow\infty$, with  $s/q$ being bounded, the upper bound on the expected error for ordered SGD (the right hand-side of Equation \ref{eq:ge_1}) is (strictly) less than that  for (mini-batch) SGD (the right hand-side of Equation \ref{eq:ge_standard}) if $\Qcal_{n}(\Theta;s,q) +L(\hat \theta_s;\Dcal) -L_{q}(\hat \theta_q;\Dcal) > 0$ or if  $L(\hat \theta_s;\Dcal) -L_{q}(\hat \theta_q;\Dcal)>0$.

For a given model  $f$, whether Theorem \ref{thm:ge_1} provides a non-vacuous bound depends on the choice of $\Theta$. In Appendix \ref{sec:app:discuss}, we discuss this effect as well as a standard way  to derive various data-dependent bounds from Theorem \ref{thm:ge_1}.

\vspace{-0.19cm}
\section{Experiments}\label{sec:nume}
\vspace{-0.19cm}

\begin{table*}[t!]
\centering \renewcommand{\arraystretch}{0.5} \fontsize{9.5pt}{9.5pt}\selectfont
\caption{Test errors (\%) of mini-batch SGD and ordered SGD (OSGD). The  last column labeled ``Improve'' shows  relative improvements (\%) from mini-batch SGD to ordered SGD. In the other columns, the numbers indicate the mean test errors (and standard deviations in parentheses) over ten random trials. The first column shows `No' for  no  data augmentation, and `Yes' for data augmentation.} \label{tbl:test_error} 
\begin{tabular}{lccccc}
\toprule
Data Aug & Datasets & Model &   mini-batch SGD & OSGD & Improve\\
\midrule
No & Semeion  & Logistic model & 10.76 (0.35) & 9.31 (0.42) & 13.48 \\
\midrule
No & MNIST    & Logistic model & 7.70 (0.06)  & 7.35 (0.04)  & 4.55 \\
\midrule
No & Semeion  & SVM & 11.05  (0.72) & 10.25 (0.51) & 7.18 \\
\midrule
No & MNIST    & SVM & 8.04 (0.05)  & 7.66 (0.07)  & 4.60 \\
\midrule
No & Semeion  & LeNet & 8.06 (0.61) & 6.09 (0.55) & 24.48 \\
\midrule
No & MNIST    & LeNet & 0.65 (0.04)  & 0.57 (0.06)  & 11.56 \\
\midrule
No & KMNIST    & LeNet & 3.74 (0.08) & 3.09 (0.14) & 17.49 \\
\midrule
No & Fashion-MNIST    & LeNet & 8.07 (0.16) & 8.03 (0.26) & 0.57 \\
\midrule
No & CIFAR-10  & PreActResNet18 & 13.75 (0.22) & 12.87 (0.32) & 6.41 \\
\midrule
No & CIFAR-100  & PreActResNet18 & 41.80 (0.40) & 41.32 (0.43) & 1.17 \\
\midrule
No & SVHN      & PreActResNet18 & 4.66 (0.10) & 4.39 (0.11) & 5.95 \\
\midrule
Yes & Semeion  & LeNet & 7.47 (1.03) & 5.06 (0.69) & 32.28 \\
\midrule
Yes & MNIST    & LeNet & 0.43 (0.03)  & 0.39 (0.03)  & 9.84 \\
\midrule
Yes &KMNIST    & LeNet & 2.59 (0.09) & 2.01 (0.13) & 22.33 \\
\midrule
Yes &Fashion-MNIST    & LeNet & 7.45 (0.07) & 6.49 (0.19) & 12.93 \\
\midrule
Yes &CIFAR-10  & PreActResNet18 & 8.08 (0.17) & 7.04 (0.12) & 12.81 \\
\midrule
Yes &CIFAR-100  & PreActResNet18 & 29.95 (0.31) & 28.31 (0.41) & 5.49 \\
\midrule
Yes &SVHN      & PreActResNet18 & 4.45 (0.07) & 4.00 (0.08) & 10.08 \\
\bottomrule
\end{tabular} 
\end{table*} 

\begin{figure*}[t!]
\begin{subfigure}[b]{0.0112\textwidth}
  \includegraphics[scale=0.16]{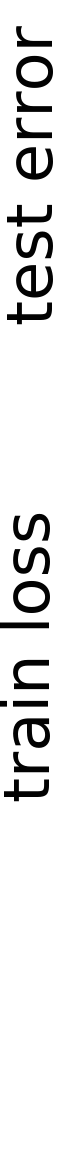}
\end{subfigure}
\begin{subfigure}[b]{0.2424\textwidth}
  \includegraphics[width=\textwidth]{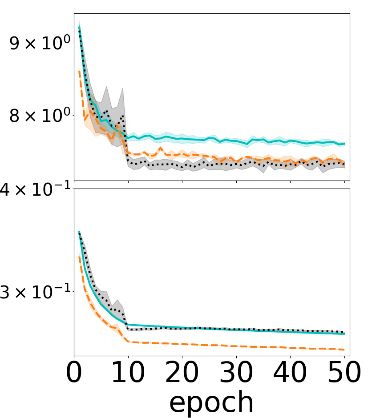}
  \caption{MNIST \& Logistic} \label{fig:plots_sgd:minist_linear_0}  
\end{subfigure}
\begin{subfigure}[b]{0.2424\textwidth}
  \includegraphics[width=\textwidth]{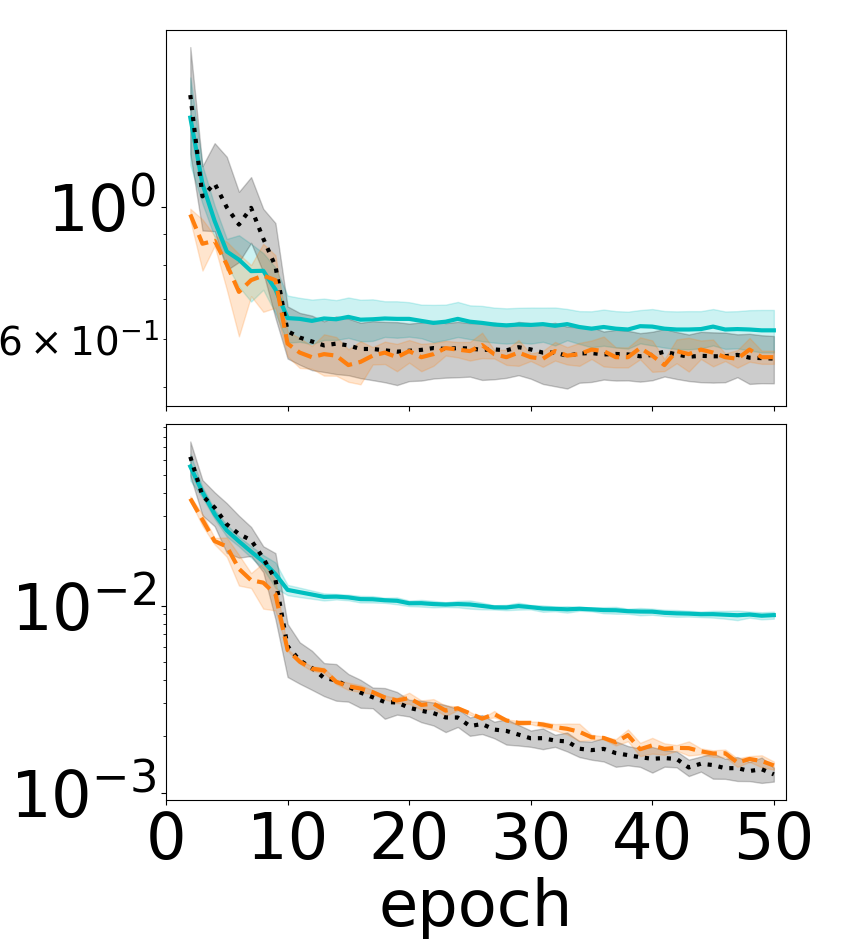}
  \caption{MNIST  \& LeNet} 
\end{subfigure}
\begin{subfigure}[b]{0.2424\textwidth}
  \includegraphics[width=\textwidth]{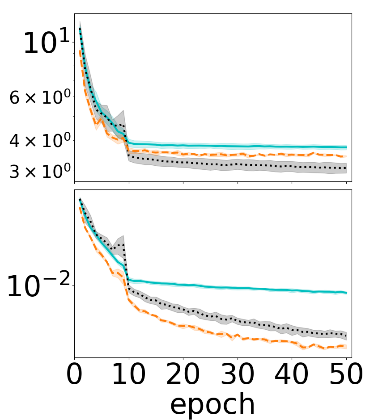}
  \caption{KMNIST} 
\end{subfigure}
\begin{subfigure}[b]{0.2424\textwidth}
  \includegraphics[width=\textwidth]{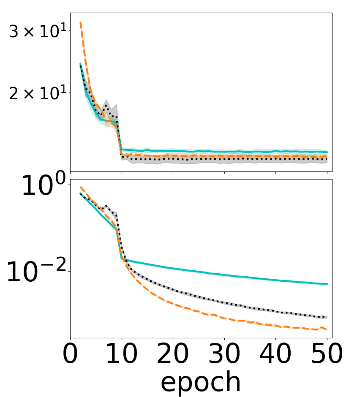}\llap{\shortstack{\includegraphics[scale=0.16]{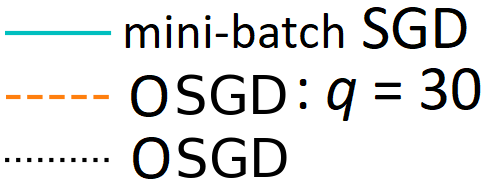}\\
        \rule{0ex}{1.38in}}\rule{0.14in}{0ex}}
  \caption{CIFAR-10} \label{fig:plots_sgd:cifar10_0}
\end{subfigure}
\begin{subfigure}[b]{0.0112\textwidth}
  \includegraphics[scale=0.16]{fig/labels/y_label}
\end{subfigure}
\begin{subfigure}[b]{0.2424\textwidth}
  \includegraphics[width=\textwidth]{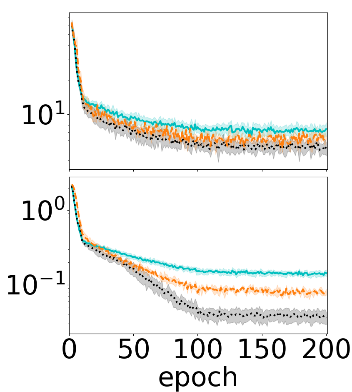}
  \caption{Semeion  \& LeNet} 
\end{subfigure}
\begin{subfigure}[b]{0.2424\textwidth}
  \includegraphics[width=\textwidth]{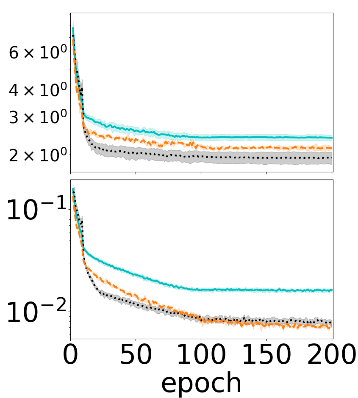}
  \caption{KMNIST}   
\end{subfigure}
\begin{subfigure}[b]{0.2424\textwidth}
  \includegraphics[width=\textwidth]{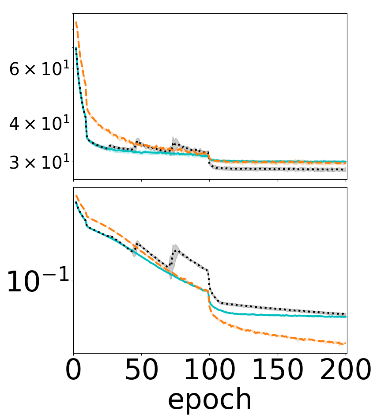}
  \caption{CIFAR-100} \label{fig:plots_sgd:cifar100_1}   
\end{subfigure}
\begin{subfigure}[b]{0.2424\textwidth}
  \includegraphics[width=\textwidth]{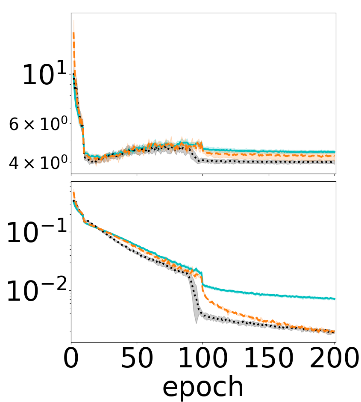}
  \caption{SVHN} 
\end{subfigure}
\caption{Test error and training loss (in log scales) versus the number of epoch. These are without data augmentation in subfigures (a)-(d), and with data augmentation in subfigures (e)-(h). The lines indicate the mean values over 10 random trials, and the shaded regions represent  intervals of the sample standard deviations. }
\label{fig:plots_sgd}
\end{figure*}

In this section, we empirically evaluate ordered SGD with various datasets, models and settings. To avoid an extra freedom due to the hyper-parameter $q$, we introduce a single fixed setup of the adaptive values of  $q$  as the default setting:    $q=s$ at the beginning of training, $q=\floor{s/2}$ once $\mathrm{train\_acc}\ge 80\%$, $q=\floor{s/4}$ once $\mathrm{train\_acc}\ge 90\%$, $q=\floor{s/8}$ once $\mathrm{train\_acc}\ge 95\%$, and $q=\floor{s/16}$ once $\mathrm{train\_acc}\ge 99.5\%$, where $\mathrm{train\_acc}$ represents  training accuracy. The value of $q$ was automatically updated at the end of each epoch based on this simple rule. 
This rule was derived based on the intuition that in the early stage of training, all samples are informative to build a rough model, while the samples around the boundary (with larger losses) are more helpful to build the final classifier in later stage. In the figures and tables of this section, we refer to ordered SGD with this rule as `OSGD', and  ordered SGD with a fixed value $q=\bar q$  as `OSGD: $q=\bar q$'.

{\bf Experiment with fixed hyper-parameters.} For this experiment, we fixed all hyper-parameters a priori across all different datasets and models by using a standard hyper-parameter setting of mini-batch SGD, instead of aiming for  state-of-the-art test errors for each dataset with a possible issue of over-fitting to  test and validation datasets \citep{dwork2015reusable,rao2008dangers}. We fixed the mini-batch size $s$ to be 64, the weight decay rate to be $10^{-4}$,  the initial learning rate to be $0.01$, and the momentum coefficient to be $0.9$. See Appendix \ref{sec:app:exp} for more details of the experimental settings. The code to reproduce all the results is publicly available at: [the link is
hidden for anonymous submission].     

Table \ref{tbl:test_error} compares the testing performance of ordered SGD and mini-batch SGD for different models and datasets. Table \ref{tbl:test_error} consistently shows that ordered SGD improved mini-batch SGD in test errors.
The table reports the mean and the standard deviation of test errors (i.e., 100 $\times$ the average of 0-1 losses on test dataset)  over $10$ random experiments with different random seeds. The table also summarises the relative improvements of ordered SGD over mini-batch SGD, which is defined as [$100 \times$ ((mean test error of mini-batch SGD) - (mean test error of ordered SGD)) / (mean test error of mini-batch SGD)]. Logistic model refers to linear multinomial logistic regression model, SVM refers to linear multiclass support vector machine,   LeNet refers to a standard variant of LeNet \citep{lecun1998gradient} with ReLU activations, and PreActResNet18 refers to pre-activation ResNet with  18 layers  \citep{he2016identity}.

Figure \ref{fig:plots_sgd} shows the test error and the average training loss of mini-batch SGD and ordered SGD versus the number of epoch.  As shown in the figure, ordered SGD with the fixed $q$ value also outperformed mini-batch SGD in general. In the figures, the reported training losses refer to the standard empirical average loss $\frac{1}{n}\sum_{i=1}^n L_i(\theta)$ measured at the end of each epoch. When compared to mini-batch SGD, ordered SGD  had  lower test errors  \textit{while  having higher training losses} in Figures \ref{fig:plots_sgd:minist_linear_0},   \ref{fig:plots_sgd:cifar10_0} and \ref{fig:plots_sgd:cifar100_1}, because ordered SGD optimizes over the ordered empirical loss instead. This is consistent with our motivation and theory of ordered SGD in Sections \ref{sec:algo},  \ref{sec:optimization} and \ref{sec:generalization}. The qualitatively similar behaviors were also  observed  with   all of the 18 various problems as shown  in Appendix \ref{sec:app:exp}.    

\begin{table}[t!]
\centering 
\captionof{table}{Average wall-clock time (seconds) per epoch with data augmentation. PreActResNet18 was used for CIFAR-10, CIFAR-100, and SVHN, while LeNet was used for MNIST and KMNIST.}  
 \label{tbl:wall_clock_time_partial}
\renewcommand{\arraystretch}{0.5} \fontsize{10pt}{10pt}\selectfont 
\vspace{-2pt} 
\begin{tabular}{lccccc}
\toprule
Datasets &  mini-batch SGD & OSGD  \\
\midrule
MNIST    &  14.44 (0.54)  & 14.77 (0.41)  \\
\midrule
KMNIST     & 12.17 (0.33) & 11.42  (0.29)  \\
\midrule
CIFAR-10  &  48.18 (0.58) & 46.40 (0.97)  \\
\midrule
CIFAR-100  &  47.37 (0.84) & 44.74 (0.91) \\
\midrule
SVHN      &  72.29 (1.23) & 67.95 (1.54)  \\
\bottomrule
\end{tabular}
\vspace{23pt}
\end{table}
\begin{figure}[t!]
\centering
\vspace{-10pt}
\includegraphics[width=0.85\columnwidth,height=0.7\columnwidth]{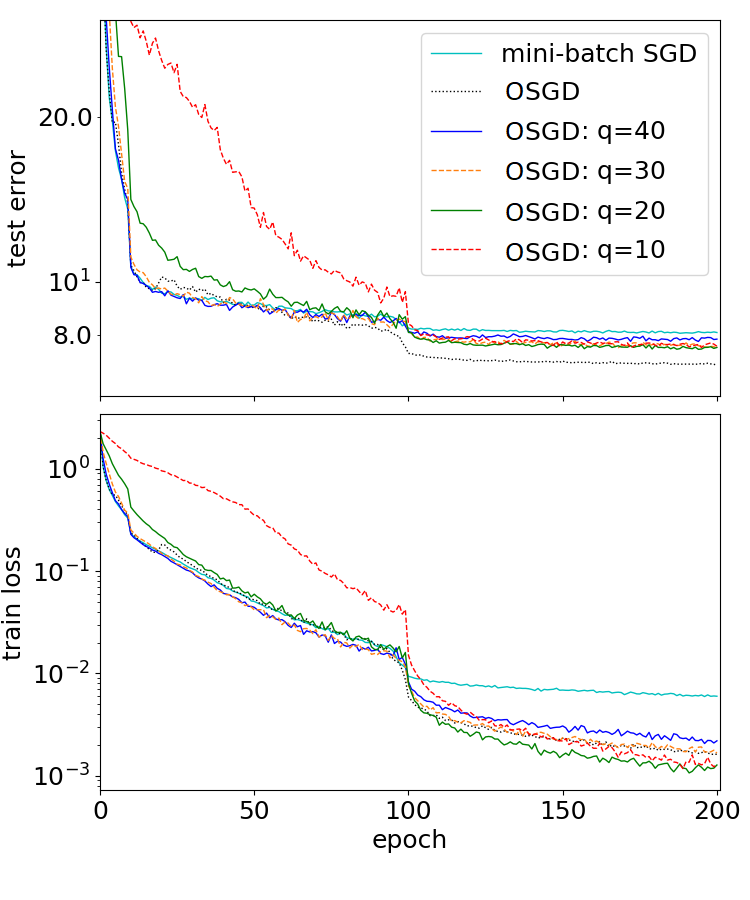}
\vspace{-0pt}
\captionof{figure}{Effect of different $q$ values with CIFAR-10.} \label{fig:plots_sgd_q_partial}
\end{figure}

Moreover, ordered SGD is a computationally efficient algorithm. Table \ref{tbl:wall_clock_time_partial} shows the wall-clock time in illustrative four experiments, whereas Table \ref{tbl:wall_clock_time} in Appendix \ref{sec:app:exp} summarizes the wall-clock time in all experiments. The wall-clock time of ordered SGD measures the time spent by all computations of ordered SGD, including the extra computation of finding  top-$q$ samples in a mini-batch (line 4 of Algorithm \ref{al:qSGD}). The extra computation is generally negligible and can be completed in  $O(s \log q )$ or $O(s)$  by using a sorting/selection algorithm.  The ordered SGD algorithm can be faster than mini-batch SGD because ordered SGD only computes the (sub-)gradient $\tg^{t}$ of the top-$q$ samples (in line 5 of Algorithm \ref{al:qSGD}).  As shown in Tables \ref{tbl:wall_clock_time_partial} and  \ref{tbl:wall_clock_time}, ordered SGD was faster than mini-batch SGD for all larger models with PreActResNet18. This is because  the computational reduction of the back-propagation in ordered SGD can dominate the small extra cost of finding  top-$q$ samples  in larger problems.

{\bf Experiment with  different $q$ values.} Figure \ref{fig:plots_sgd_q_partial} shows the effect of different fixed $q$ values for CIFAR-10  with PreActResNet18.
Ordered SGD improved the test errors of mini-batch SGD with different fixed   $q$ values. We also report the same observation with different datasets and models in Appendix \ref{sec:app:exp}. 

{\bf Experiment with different learning rates and mini-batch sizes.} Figures \ref{fig:new1} and  \ref{fig:new2} in  Appendix \ref{sec:app:exp}  consistently show the improvement of ordered SGD over mini-batch SGD with different  different learning rates and mini-batch sizes.

{\bf Experiment with the best learning rate, mixup, and random erasing.} Table \ref{tbl:other_data_aug} summarises the experimental results with the data augmentation methods of random erasing (RE) \citep{zhong2017random}  and mixup  \citep{zhang2017mixup,verma2019manifold} by using CIFAR-10 dataset.
For this experiment, we  purposefully adopted the  setting that  favors mini-batch SGD. That is, for both mini-batch SGD and ordered SGD, we used  hyper-parameters  tuned for mini-batch SGD. For RE and mixup data, we used the  same tuned hyper-parameter settings (including learning rates) and the codes as those  in the previous studies that used  mini-batch SGD   \citep{zhong2017random,verma2019manifold} (with WRN-28-10 for RE and with PreActResNet18 for mixup). For standard
data augmentation,
we 
first searched the best learning rate of mini-batch SGD based on  the test error (purposefully overfitting to the test dataset for mini-batch\ SGD) by using the grid search with learning rates of 1.0, 0.5, 0.1, 0.05. 0.01, 0.005, 0.001, 0.0005, 0.0001. Then, we used the best learning rate of mini-batch SGD
for
ordered SGD (instead of using the best learning rate of ordered SGD for ordered SGD).
As shown in Table \ref{tbl:other_data_aug},  ordered SGD with hyper-parameters tuned for mini-batch SGD  still outperformed fine-tuned mini-batch SGD with the different data augmentation methods.

\begin{table}[t!]
\centering 
\renewcommand{\arraystretch}{0.5} \fontsize{10pt}{10pt}\selectfont
\captionof{table}{Test errors (\%)  by  using the best learning rate of mini-batch SGD   with   various data augmentation methods  for CIFAR-10.} \label{tbl:other_data_aug}
\vspace{-8pt} 
\begin{tabular}{lcccc}
\toprule
Data Aug & mini-batch SGD & OSGD & Improve\\
\midrule
Standard   & 6.94 & 6.46 & 6.92 \\
\midrule
RE   & 3.24 & 3.06 & 5.56 \\
\midrule
Mixup    &  3.31  & 3.05  & 7.85 \\
\bottomrule
\end{tabular} 
\end{table}

\vspace{-0.2cm}
\section{Related work and extension}
\vspace{-0.2cm}
Although there is no direct predecessor of our work, 
 the following  fields are  related to this paper.

{\bf Other mini-batch stochastic methods}. The proposed sampling strategy and our theoretical analyses are generic and can be extended to other (mini-batch) stochastic methods, including Adam \citep{kingma2014adam}, stochastic mirror descent  \citep{beck2003mirror,nedic2014stochastic,lu2017relative,lu2018accelerating,zhang2018convergence}, and proximal stochastic subgradient methods \citep{davis2018stochastic}.  Thus, our results open up the research direction for further studying the proposed stochastic optimization framework with different base algorithms such as Adam and  AdaGrad. To illustrate it, we presented ordered Adam and reported the numerical results in Appendix \ref{sec:app:exp}.

{\bf Importance Sampling SGD. }Stochastic gradient descent with importance sampling has been an active research area for the past several years \citep{needell2014stochastic,zhao2015stochastic,alain2015variance,loshchilov2015online,gopal2016adaptive,katharopoulos2018not}. In the convex setting, \citep{zhao2015stochastic, needell2014stochastic} show that the optimal sampling distribution for minimizing $L(\theta)$ is proportional to the per-sample gradient norm. However, maintaining the norm of gradient for individual samples can be computationally expensive when the dataset size $n$ or the parameter vector size $d_{\theta}$  is large in particular for many applications of deep learning. These importance sampling methods are inherently different from ordered SGD in that importance sampling is used to reduce the number of iterations for minimizing $L(\theta)$, whereas ordered SGD is designed to learn a different type of models by minimizing the new objective function $L_q(\theta)$.

{\bf Average Top-k Loss. }The average top-$k$ loss is introduced by \citet{fan2017learning} as an alternative to the empirical average loss $L(\theta)$. The ordered loss function $L_q(\theta)$  differs from the  average top-$k$ loss as shown in Section   \ref{sec:optimization}. Furthermore, our proposed framework is fundamentally different from the average top-$k$ loss. First, the  algorithms are  different -- the stochastic method proposed in \citet{fan2017learning} utilizes duality of the  function and is unusable  for deep neural networks (and other non-convex problems), while our proposed method is a modification of mini-batch SGD that is usable for deep neural networks (and other non-convex problems) and scales well for large problems.  Second, the optimization results are different, and in particular, the objective functions are different and we have convergence analysis for weakly convex (non-convex) functions.
Finally, the focus of generalization property is different --  \citet{fan2017learning} focuses on the calibration for binary classification problem, while we focus on the generalization bound that works for general classification and regression problems.

{\bf Random-then-Greedy Procedure. \ }Ordered SGD randomly picks a subset of samples and then greedily utilizes a part of the subset, which is related to the random-then-greedy procedure proposed recently in the different topic -- the greedy weak learner for gradient boosting \citep{lu2018randomized}.

\vspace{-0.2cm}
\section{Conclusion}
\vspace{-0.2cm}
We have presented an efficient stochastic first-order method, ordered SGD, for learning an effective predictor in machine learning problems. We have shown that ordered SGD minimizes a new ordered empirical loss $L_q(\theta)$, based on which we have developed the optimization and generalization properties of ordered SGD. The numerical experiments confirmed the effectiveness of our proposed algorithm.


\bibliography{all}

\newcommand{\noopsort}[1]{} \newcommand{\printfirst}[2]{#1}
  \newcommand{\singleletter}[1]{#1} \newcommand{\switchargs}[2]{#2#1}
\begin{thebibliography}{}

\bibitem[Abadi et~al., 2016]{abadi2016tensorflow}
Abadi, M., Barham, P., Chen, J., Chen, Z., Davis, A., Dean, J., Devin, M.,
  Ghemawat, S., Irving, G., Isard, M., et~al. (2016).
\newblock Tensorflow: A system for large-scale machine learning.
\newblock In {\em 12th $\{$USENIX$\}$ Symposium on Operating Systems Design and
  Implementation ($\{$OSDI$\}$ 16)}, pages 265--283.

\bibitem[Alain et~al., 2015]{alain2015variance}
Alain, G., Lamb, A., Sankar, C., Courville, A., and Bengio, Y. (2015).
\newblock Variance reduction in sgd by distributed importance sampling.
\newblock {\em arXiv preprint arXiv:1511.06481}.

\bibitem[Allen-Zhu, 2017]{allen2017natasha}
Allen-Zhu, Z. (2017).
\newblock Natasha: Faster non-convex stochastic optimization via strongly
  non-convex parameter.
\newblock In {\em Proceedings of the 34th International Conference on Machine
  Learning-Volume 70}, pages 89--97. JMLR. org.

\bibitem[Bartlett et~al., 2017]{bartlett2017spectrally}
Bartlett, P.~L., Foster, D.~J., and Telgarsky, M.~J. (2017).
\newblock Spectrally-normalized margin bounds for neural networks.
\newblock In {\em Advances in Neural Information Processing Systems}, pages
  6240--6249.

\bibitem[Bartlett and Mendelson, 2002]{bartlett2002rademacher}
Bartlett, P.~L. and Mendelson, S. (2002).
\newblock Rademacher and gaussian complexities: Risk bounds and structural
  results.
\newblock {\em Journal of Machine Learning Research}, 3(Nov):463--482.

\bibitem[Beck and Teboulle, 2003]{beck2003mirror}
Beck, A. and Teboulle, M. (2003).
\newblock Mirror descent and nonlinear projected subgradient methods for convex
  optimization.
\newblock {\em Operations Research Letters}, 31(3):167--175.

\bibitem[Bottou et~al., 2018]{bottou2018optimization}
Bottou, L., Curtis, F.~E., and Nocedal, J. (2018).
\newblock Optimization methods for large-scale machine learning.
\newblock {\em Siam Review}, 60(2):223--311.

\bibitem[Boyd and Mutapcic, 2008]{boyd2008stochastic}
Boyd, S. and Mutapcic, A. (2008).
\newblock Stochastic subgradient methods.
\newblock {\em Lecture Notes for EE364b, Stanford University}.

\bibitem[Davis and Drusvyatskiy, 2018]{davis2018stochastic}
Davis, D. and Drusvyatskiy, D. (2018).
\newblock Stochastic subgradient method converges at the rate {$O (k^{-1/4})$}
  on weakly convex functions.
\newblock {\em arXiv preprint arXiv:1802.02988}.

\bibitem[Dean et~al., 2012]{dean2012large}
Dean, J., Corrado, G., Monga, R., Chen, K., Devin, M., Mao, M., Senior, A.,
  Tucker, P., Yang, K., Le, Q.~V., et~al. (2012).
\newblock Large scale distributed deep networks.
\newblock In {\em Advances in neural information processing systems}, pages
  1223--1231.

\bibitem[Dwork et~al., 2015]{dwork2015reusable}
Dwork, C., Feldman, V., Hardt, M., Pitassi, T., Reingold, O., and Roth, A.
  (2015).
\newblock The reusable holdout: Preserving validity in adaptive data analysis.
\newblock {\em Science}, 349(6248):636--638.

\bibitem[Fan et~al., 2017]{fan2017learning}
Fan, Y., Lyu, S., Ying, Y., and Hu, B. (2017).
\newblock Learning with average top-k loss.
\newblock In {\em Advances in Neural Information Processing Systems}, pages
  497--505.

\bibitem[Gopal, 2016]{gopal2016adaptive}
Gopal, S. (2016).
\newblock Adaptive sampling for {SGD} by exploiting side information.
\newblock In {\em International Conference on Machine Learning}, pages
  364--372.

\bibitem[He et~al., 2016]{he2016identity}
He, K., Zhang, X., Ren, S., and Sun, J. (2016).
\newblock Identity mappings in deep residual networks.
\newblock In {\em European Conference on Computer Vision}, pages 630--645.
  Springer.

\bibitem[Katharopoulos and Fleuret, 2018]{katharopoulos2018not}
Katharopoulos, A. and Fleuret, F. (2018).
\newblock Not all samples are created equal: Deep learning with importance
  sampling.
\newblock In {\em International Conference on Machine Learning}, pages
  2530--2539.

\bibitem[Kawaguchi et~al., 2017]{kawaguchi2017generalization}
Kawaguchi, K., Kaelbling, L.~P., and Bengio, Y. (2017).
\newblock Generalization in deep learning.
\newblock {\em arXiv preprint arXiv:1710.05468}.

\bibitem[Kingma and Ba, 2014]{kingma2014adam}
Kingma, D.~P. and Ba, J. (2014).
\newblock Adam: A method for stochastic optimization.
\newblock {\em arXiv preprint arXiv:1412.6980}.

\bibitem[LeCun et~al., 1998]{lecun1998gradient}
LeCun, Y., Bottou, L., Bengio, Y., and Haffner, P. (1998).
\newblock Gradient-based learning applied to document recognition.
\newblock {\em Proceedings of the IEEE}, 86(11):2278--2324.

\bibitem[Loshchilov and Hutter, 2015]{loshchilov2015online}
Loshchilov, I. and Hutter, F. (2015).
\newblock Online batch selection for faster training of neural networks.
\newblock {\em arXiv preprint arXiv:1511.06343}.

\bibitem[Lu, 2017]{lu2017relative}
Lu, H. (2017).
\newblock " relative-continuity" for non-lipschitz non-smooth convex
  optimization using stochastic (or deterministic) mirror descent.
\newblock {\em arXiv preprint arXiv:1710.04718}.

\bibitem[Lu et~al., 2018]{lu2018accelerating}
Lu, H., Freund, R., and Mirrokni, V. (2018).
\newblock Accelerating greedy coordinate descent methods.
\newblock In {\em International Conference on Machine Learning}, pages
  3263--3272.

\bibitem[Lu and Mazumder, 2018]{lu2018randomized}
Lu, H. and Mazumder, R. (2018).
\newblock Randomized gradient boosting machine.
\newblock {\em arXiv preprint arXiv:1810.10158}.

\bibitem[Mohri et~al., 2012]{mohri2012foundations}
Mohri, M., Rostamizadeh, A., and Talwalkar, A. (2012).
\newblock {\em Foundations of machine learning}.
\newblock MIT press.

\bibitem[Nedic and Lee, 2014]{nedic2014stochastic}
Nedic, A. and Lee, S. (2014).
\newblock On stochastic subgradient mirror-descent algorithm with weighted
  averaging.
\newblock {\em SIAM Journal on Optimization}, 24(1):84--107.

\bibitem[Needell et~al., 2014]{needell2014stochastic}
Needell, D., Ward, R., and Srebro, N. (2014).
\newblock Stochastic gradient descent, weighted sampling, and the randomized
  kaczmarz algorithm.
\newblock In {\em Advances in Neural Information Processing Systems}, pages
  1017--1025.

\bibitem[Paszke et~al., 2017]{paszke2017automatic}
Paszke, A., Gross, S., Chintala, S., Chanan, G., Yang, E., DeVito, Z., Lin, Z.,
  Desmaison, A., Antiga, L., and Lerer, A. (2017).
\newblock Automatic differentiation in pytorch.
\newblock In {\em Autodiff Workshop at Conference on Neural Information
  Processing Systems}.

\bibitem[Rao et~al., 2008]{rao2008dangers}
Rao, R.~B., Fung, G., and Rosales, R. (2008).
\newblock On the dangers of cross-validation. an experimental evaluation.
\newblock In {\em Proceedings of the 2008 SIAM international conference on data
  mining}, pages 588--596. SIAM.

\bibitem[Rockafellar and Wets, 2009]{rockafellar2009variational}
Rockafellar, R.~T. and Wets, R. J.-B. (2009).
\newblock {\em Variational analysis}, volume 317.
\newblock Springer Science \& Business Media.

\bibitem[Seide and Agarwal, 2016]{seide2016cntk}
Seide, F. and Agarwal, A. (2016).
\newblock Cntk: Microsoft's open-source deep-learning toolkit.
\newblock In {\em Proceedings of the 22nd ACM SIGKDD International Conference
  on Knowledge Discovery and Data Mining}, pages 2135--2135. ACM.

\bibitem[Shalev-Shwartz and Wexler, 2016]{shalev2016minimizing}
Shalev-Shwartz, S. and Wexler, Y. (2016).
\newblock Minimizing the maximal loss: How and why.
\newblock In {\em International Conference on Machine Learning}, pages
  793--801.

\bibitem[Verma et~al., 2019]{verma2019manifold}
Verma, V., Lamb, A., Beckham, C., Najafi, A., Mitliagkas, I., Lopez-Paz, D.,
  and Bengio, Y. (2019).
\newblock Manifold mixup: Better representations by interpolating hidden
  states.
\newblock In {\em International Conference on Machine Learning}, pages
  6438--6447.

\bibitem[Weston et~al., 1999]{weston1999support}
Weston, J., Watkins, C., et~al. (1999).
\newblock Support vector machines for multi-class pattern recognition.
\newblock In {\em Esann}, volume~99, pages 219--224.

\bibitem[Zhang et~al., 2017]{zhang2017mixup}
Zhang, H., Cisse, M., Dauphin, Y.~N., and Lopez-Paz, D. (2017).
\newblock mixup: Beyond empirical risk minimization.
\newblock {\em arXiv preprint arXiv:1710.09412}.

\bibitem[Zhang and He, 2018]{zhang2018convergence}
Zhang, S. and He, N. (2018).
\newblock On the convergence rate of stochastic mirror descent for nonsmooth
  nonconvex optimization.
\newblock {\em arXiv preprint arXiv:1806.04781}.

\bibitem[Zhao and Zhang, 2015]{zhao2015stochastic}
Zhao, P. and Zhang, T. (2015).
\newblock Stochastic optimization with importance sampling for regularized loss
  minimization.
\newblock In {\em international conference on machine learning}, pages 1--9.

\bibitem[Zhong et~al., 2017]{zhong2017random}
Zhong, Z., Zheng, L., Kang, G., Li, S., and Yang, Y. (2017).
\newblock Random erasing data augmentation.
\newblock {\em arXiv preprint arXiv:1708.04896}.

\end{thebibliography}

\clearpage
\onecolumn

\appendix

\begin{center}
\textbf{\LARGE 
 Appendix
\vspace{5pt}}
\end{center}

\section{Proofs} \label{sec:app:proof}
In Appendix \ref{sec:app:proof}, we provide complete proofs of the theoretical results.   
\subsection{Proof of Theorem \ref{thm:equiv}}
\begin{proof}
We just need to show that $\tg$ is an unbiased estimator of a sub-gradient of $L_q(\theta)$ at $\theta^t$, namely $\EE \tg \in \partial L_q(\theta^t)$.

At first, it holds that
$$\EE \tg^t = \frac{1}{q} \EE \sum_{i\in Q} g_i^t + g_R^t = \frac{1}{q} \sum_{i=1}^n P(i\in Q) g_i^t + g_R^t= \frac{1}{q} \sum_{j=1}^n P((j)\in Q) g_{(j)}^t + g_R^t\ ,$$
where $g_i^t\in \partial L_i(\theta^t)$ is a sub-gradient of $L_i$ at $\theta^t$ and $g_R^t\in \partial R(\theta^t)$.
In the above equality chain, the third equality is simply the definition of expectation, and the last equality is because $((1),(2),\ldots, (n))$ is a permutation of $(1,2,\ldots, n)$. 

For any given index $j$, define $A_j=((1), (2), \ldots, (j-1))$, then
\begin{equation}\label{eq:prob-0}
    \begin{array}{ll}
        P((j)\in Q)     & = P\left( (j)\in \qargmax_{i\in S} L_i(\theta)\right)  \\
         & =  P\left( (j)\in S \text{ and } S \text{ contains at most } q-1 \text{ items in } A_j\right) \\
         & = P\left( (j)\in S \right)  P\left(S \text{ contains at most } q-1 \text{ items in } A_j | (j)\in S \right) \\
         & = P\left( (j)\in S \right) \sum_{l=0}^{q-1}  P\left(S \text{ contains } l \text{ items in } A_j | (j)\in S \right) .\\
    \end{array}
\end{equation}
Notice that $S$ is randomly chosen from sample index set $(1,2,\ldots,n)$ without replacement. There are in total $\binom{n}{s}$ different sets $S$ such that $|S|=s$. Among them, there are $\binom{n-1}{s-1}$ different sets $S$ which contains the index $(j)$, thus 
\begin{equation}\label{eq:prob-1}
    P\left( (j)\in S \right)= \frac{\binom{n-1}{s-1}}{\binom{n}{s}} \ .
\end{equation}
Given the condition $(j)\in S$, $S$ contains $l$ items in $A_j$ means $S$ contains $s-l-1$ items in $\{(j+1), (j+2) \ldots, (n)\}$, thus there are $\binom{j-1}{l}\binom{n-j}{s-l-1}$ such possible set $S$, whereby it holds that
\begin{equation}\label{eq:prob-2}
   P\left(S \text{ contains } l \text{ items in } A_j | (j)\in S \right) = \frac{\binom{j-1}{l}\binom{n-j}{s-l-1}}{\binom{n-1}{s-1}} \ . 
\end{equation}
Substituting Equations \eqref{eq:prob-1} and \eqref{eq:prob-2} into Equation \eqref{eq:prob-0}, we arrive at
\begin{equation*}
   P((j)\in T) = \frac{\binom{n-1}{s-1}}{\binom{n}{s}} \sum_{l=0}^{q-1} \frac{\binom{j-1}{l}\binom{n-j}{s-l-1}}{\binom{n-1}{s-1}} = \frac{\sum_{l=0}^{q-1}\binom{j-1}{l}\binom{n-j}{s-l-1}}{\binom{n}{s}} = \gamma_j \ .
\end{equation*}
Therefore,
$$
\EE \tg^t = \frac{1}{q} \sum_{j=1}^n P((j)\in Q) g^t_{(j)}  + g_R^t = \frac{1}{q} \sum_{j=1}^n \gamma_j g^t_{(j)}  + g_R^t \in \partial L_q(\theta^t) \ ,
$$
where the last inequality is due to the aditivity of sub-gradient (for both convex and weakly convex function)
\end{proof}

\subsection{Proof of Proposition \ref{prop:limit}}
We just need to show that 
\begin{equation}\label{eq:asymp_gamma}
    \lim_{j, n\rightarrow \infty, j/n=z} \gamma_j = \sum_{l=0}^{q-1}\frac{1}{n}\left(\frac{j}{n}\right)^{l}\left(\frac{n-j}{n}\right)^{s-l-1}\frac{s!}{l!(s-l-1)!} \ ,
\end{equation}
then we finish the proof by changing variable $z=\frac{j}{n}$. 

At first, the Stirling's approximation yields that when $n$ and $j$ are both sufficiently large, it holds that
\begin{equation}\label{eq:stirling}
\binom{n}{j} \sim \sqrt{\frac{n}{2\pi j (n-j)}} \frac{n^n}{j^j(n-j)^{n-j}} \ .    
\end{equation}

Thus,
\begin{equation}\label{eq:prop1_mid1}
    \lim_{j,n\rightarrow \infty, j/n=z}\frac{\binom{n-s}{j-1-l}}{\binom{n-1}{j-1}} = \frac{\frac{n^{n-s}}{j^{j-1-l}(n-j)^{n-j-s+1+l}}}{\frac{n^{n-1}}{j^{j-1}(n-j)^{n-j}}} = \frac{j^l(n-j)^{s-l-1}}{n^{s-1}} = \left(\frac{j}{n}\right)^{l} \left(\frac{n-j}{n}\right)^{s-l-1} ,
\end{equation}
where the first equality utilize Equation \eqref{eq:stirling} and the fact that $s, l, 1$ are negligible in the limit case (except the exponent terms).

On the other hand, it holds by rearranging the factorial numbers that
\begin{equation}\label{eq:prop1_mid2}
    \frac{1}{n}\frac{\binom{n-s}{j-1-l}}{\binom{n-1}{j-1}}\frac{s!}{l!(s-l-1)!} = \frac{\binom{j-1}{l}\binom{n-j}{s-l-1}}{\binom{n}{s}} \ .
\end{equation}

Combining Equations \eqref{eq:prop1_mid1} and \eqref{eq:prop1_mid2} and summing $l$, we arrive at Equation \eqref{eq:asymp_gamma}.



By noticing $s>q$, it holds that
\begin{align*}
\frac{d}{dz}\gamma(z) & =\sum_{l=1}^{q-1}lz^{l-1}(1-z)^{s-l-1}\frac{s!}{l!(s-l-1)!}-\sum_{l=0}^{q-1}(s-l-1)z^{l}(1-z)^{s-l-2}\frac{s!}{l!(s-l-1)!}\\
 & =\sum_{l=1}^{q-1}z^{l-1}(1-z)^{s-l-1}\frac{s!}{(l-1)!(s-l-1)!}-\sum_{l=0}^{q-1}z^{l}(1-z)^{s-l-2}\frac{s!}{l!(s-l-2)!}\\
 & =\sum_{l=0}^{q-2}z^{l}(1-z)^{s-l-2}\frac{s!}{l!(s-l-2)!}-\sum_{l=0}^{q-1}z^{l}(1-z)^{s-l-2}\frac{s!}{l!(s-l-2)!}\\
 & =-z^{q-1}(1-z)^{s-q-1}\frac{s!}{l!(s-l-2)!}\\
 & \propto-z^{q-1}(1-z)^{s-q-1}.
\end{align*}
In other word, $1-\frac{1}{s}\gamma(z)$ is the cumulative of Beta($q,s-q$)
when $n\rightarrow\infty$.

\subsection{Proof of Theorem \ref{cor:opt}}
\begin{proof}
Notice that $\tg^t$ is a sub-gradient of $L_Q(\theta^t)$ where $L_{Q}(\theta^t)= \frac{1}{q} \sum_{i\in Q}  L_i(\theta^t) +  R(\theta^{t})$. Suppose $\tg^t=\frac{1}{q} \sum_{i\in Q} g_i(\theta^t)+g_R(\theta^t)$ where $g_i(\theta^t)$ is a sub-gradient of $L_i(\theta^t)$ and $g_R(\theta^t)$ is a sub-gradient of $R(\theta^t)$. Then
\begin{equation} \label{eq:cor_sub}
    \|\tg^t\|^2 = \left\|\frac{1}{q} \sum_{i\in Q} g_i(\theta^t)+g_R(\theta^t)\right\|^2\le  2\left(\left\|\frac{1}{q} \sum_{i\in Q} g_i(\theta^t)\right\|^2+\left\|g_R(\theta^t)\right\|^2\right)\le 2(G_1^2+G_2^2) \ .
\end{equation}

Meanwhile, it follows Theorem \ref{thm:equiv} that $\tg^t$ is an unbiased estimator of a sub-gradient of $L_q(\theta^t)$. Together with  Equation  \eqref{eq:cor_sub}, we obtain the statement (1) by the analysis of convex stochastic sub-gradient descent in \cite{boyd2008stochastic}.

Furthermore, suppose $L_i(\theta)+\frac{\rho}{2}\|\theta\|^2$ is convex for any $i$, then $L_q(\theta)+\frac{\rho}{2}\|\theta\|^2=\frac{1}{q}\sum_{j=1}^n\gamma_j \left(L_{(j)}(\theta)+\frac{\rho}{2}\|\theta\|^2\right)+R(\theta)$ is also convex, whereby $L_q(\theta)$ is $\rho$-weakly convex. We obtain the statement (2) by substituting into Theorem 2.1 in \cite{davis2018stochastic}. 
\end{proof}

\subsection{Proof of Theorem \ref{thm:ge_1}}

Before proving Theorem \ref{thm:ge_1}, we first show the following proposition, which gives an upper bound for $\gamma_j$:
\begin{proposition} \label{prop:opt_2}
For any $j \in \{1,\dots,n\}$, $\gamma_{j} \le\frac{s}{n}$. 
\end{proposition}
\begin{proof}
The value of $\gamma_{j}$ is equal to the probability of ordered SGD choosing the $j$-th sample in the ordered sequence $(L_{(1)}(\theta;\Dcal),\dots,L_{(n)}(\theta;\Dcal))$, which is at most the probability of mini-batch SGD choosing the $j$-th sample. The probability of mini-batch SGD choosing the $j$-th sample is $\frac{s}{n}$.
\end{proof}

We are now ready to prove Theorem \ref{thm:ge_1} by finding an upper  bound on $\sup_{\theta \in \Theta} \EE_{(x,y)}[\ell(f(x;\theta),y_{})]-L_{q}(\theta;\mathcal{D})$ based on  McDiarmid's inequality.  

\emph{Proof of Theorem \ref{thm:ge_1}}.\ \
  Define $\Phi(\Dcal)= \sup_{\theta \in \Theta} \EE_{(x,y)}[\ell(f(x;\theta),y_{})]-L_{q}(\theta;\mathcal{D})$. In this proof, our objective is to provide the upper bound on $\Phi(\Dcal)$ by using McDiarmid's inequality. To apply McDiarmid's inequality to $\Phi(\Dcal)$, we first show that $\Phi(\Dcal)$ satisfies the remaining condition of McDiarmid's inequality. Let $\Dcal$ and $\Dcal'$ be two datasets differing by exactly one point of an arbitrary index $i_{0}$; i.e.,  $\Dcal_i= \Dcal'_i$ for all $i\neq i_{0}$ and $\Dcal_{i_{0}} \neq \Dcal'_{i_{0}}$. Then, we provide an upper bound on $\Phi(\Dcal') - \Phi(\Dcal)$ as follows:  
\begin{align*}
\Phi(\Dcal') - \Phi(\Dcal)
& \le \sup_{\theta \in \Theta}   L_{q}(\theta;\mathcal{D})-L_{q}(\theta;\mathcal{D'}).
\\ & = \sup_{\theta \in \Theta} \frac{1}{q} \sum_{j=1}^n \gamma_{j} (L_{(j)}(\theta;\Dcal)-L_{(j)}(\theta;\Dcal')) 
\\ & \le \sup_{\theta \in \Theta} \frac{1}{q} \sum_{j=1}^n |\gamma_{j}||L_{(j)}(\theta;\Dcal)-L_{(j)}(\theta;\Dcal')|
\\ & \le\sup_{\theta \in \Theta} \frac{1}{q} \frac{s}{n} \sum_{j=1}^n |L_{(j)}(\theta;\Dcal)-L_{(j)}(\theta;\Dcal')|  
\end{align*}
where the first line follows the property of the supremum, $\sup (a) - \sup (b)\le\sup(a-b)$, the second line follows the definition of $L_{q}$, and the last line follows Proposition \ref{prop:opt_2} ($|\gamma_{j}| \le\frac{s}{n} $). 

We now bound the last term $ \sum_{j=1}^n |L_{(j)}(\theta;\Dcal)-L_{(j)}(\theta;\Dcal')|$. This requires a careful examination because  $|L_{(j)}(\theta;\Dcal)-L_{(j)}(\theta;\Dcal')| \neq 0$ for more than one index $j$ (although  $\Dcal$ and $\Dcal'$  differ only by exactly one point). This is because  it is possible to have  $(j;\Dcal ) \neq (j; \Dcal')$ for many indexes $j$ where $(j;\Dcal) = (j)$ in $L_{(j)}(\theta;\Dcal)$ and $(j;\Dcal') = (j)$ in $L_{(j)}(\theta;\Dcal')$.          
To analyze this effect, we now conduct case analysis. Define $l(i_{};\Dcal)$ such that $(j)=i$ where $j=l(i_{};\Dcal)$; i.e., $L_i(\theta;\Dcal)=L_{(l(i_{};\Dcal))}(\theta;\Dcal)$. 

Consider the case where $l(i_{0};\Dcal') \ge l(i_{0};\Dcal)$. Let $j_1=l(i_{0};\Dcal)$ and $j_2=l(i_{0};\Dcal') $. Then,   
\begin{align*}
\sum_{j=1}^n |L_{(j)}(\theta;\Dcal)-L_{(j)}(\theta;\Dcal')| 
& = \sum_{j=j_1}^{j_2-1} |L_{(j)}(\theta;\Dcal)-L_{(j)}(\theta;\Dcal')|+|L_{(j_{2})}(\theta;\Dcal)-L_{(j_{2})}(\theta;\Dcal')|   
\\ & =  \sum_{j=j_1}^{j_2-1} |L_{(j)}(\theta;\Dcal)-L_{(j+1)}(\theta;\Dcal)|+|L_{(j_{2})}(\theta;\Dcal)-L_{(j_{2})}(\theta;\Dcal')| \\ & =  \sum_{j=j_1}^{j_2-1} (L_{(j)}(\theta;\Dcal)-L_{(j+1)}(\theta;\Dcal))+L_{(j_{2})}(\theta;\Dcal)-L_{(j_{2})}(\theta;\Dcal')
\\ & =  L_{(j_{1})}(\theta;\Dcal)-L_{(j_{2})}(\theta;\Dcal')
\\ & \le M,
\end{align*}
where the first line uses the fact that $j_2 =l(i_{0};\Dcal')\ge l(i_{0};\Dcal)= j_1$ where $i_{0}$ is the index of samples differing in  $\Dcal$ and $\Dcal'$. The second line follows the equality $(j;\Dcal')=(j+1;\Dcal)$ from $j_1$ to $j_2-1$ in this case. The third line follows the definition of the ordering of the indexes. The fourth line follows the cancellations of the terms from the third line.        

Consider the case where $l(i_{0};\Dcal') < l(i_{0};\Dcal)$. Let $j_1=l(i_{0};\Dcal')$ and $j_2= l(i_{0};\Dcal)$. Then,    
 \begin{align*}
\sum_{j=1}^n |L_{(j)}(\theta;\Dcal)-L_{(j)}(\theta;\Dcal')| 
& =|L_{(j_{1})}(\theta;\Dcal)-L_{(j_{1})}(\theta;\Dcal')|+ \sum_{j=j_1+1}^{j_2} |L_{(j)}(\theta;\Dcal)-L_{(j)}(\theta;\Dcal')|   
\\ & =  |L_{(j_{1})}(\theta;\Dcal)-L_{(j_{1})}(\theta;\Dcal')|+\sum_{j=j_1+1}^{j_2} |L_{(j)}(\theta;\Dcal)-L_{(j-1)}(\theta;\Dcal)| \\ & =  L_{(j_{1})}(\theta;\Dcal)-L_{(j_{1})}(\theta;\Dcal')+\sum_{j=j_1+1}^{j_2} (L_{(j)}(\theta;\Dcal)-L_{(j-1)}(\theta;\Dcal))
\\ & =  L_{(j_{1})}(\theta;\Dcal')-L_{(j_{2})}(\theta;\Dcal)
\\ & \le M.
\end{align*}
where the first line uses the fact that $j_1 =l(i_{0};\Dcal')< l(i_{0};\Dcal)= j_2$ where $i_{0}$ is the index of samples differing in  $\Dcal$ and $\Dcal'$. The second line follows the equality $(j;\Dcal')=(j-1;\Dcal)$ from $j_1+1$ to $j_2$ in this case. The third line follows the definition of the ordering of the indexes. The fourth line follows the cancellations of the terms from the third line. 

Therefore, in both cases of $l(i_{0};\Dcal') \ge l(i_{0};\Dcal)$ and $l(i_{0};\Dcal') < l(i_{0};\Dcal)$, we have that 
$$
\Phi(\Dcal') - \Phi(\Dcal) \le  \frac{s}{q} \frac{M}{n}.
$$
Similarly,  $\Phi(\Dcal) - \Phi(\Dcal')\le  \frac{s}{q} \frac{M}{n}$, and hence $|\Phi(\Dcal) - \Phi(\Dcal')| \le  \frac{s}{q} \frac{M}{n}$. Thus, by McDiarmid's inequality, for any $\delta>0$, with probability at least $1-\delta$,
$$
\Phi(\Dcal) \le  \EE_{\bar \Dcal}[\Phi(\bar \Dcal)] + \frac{Ms}{q} \sqrt{\frac{\ln(1/\delta)}{2n}}.
$$ 
Moreover, since\begin{align*}
\sum_{i=1}^n r_i(\theta;\Dcal) L_{i}(\theta;\Dcal)=\sum_{j=1}^n  \gamma_{j}\sum_{i=1}^{n} \mathbbm{1}\{i=(j)\} L_{i}(\theta;\Dcal)
=\sum_{j=1}^n  \gamma_{j}L_{(j)}(\theta;\Dcal),
\end{align*}we have that
$$
L_{q}(\theta;\Dcal)=\frac{1}{q} \sum_{i=1}^n r_i(\theta;\Dcal) L_{i}(\theta;\Dcal)+R(\theta).
$$

Therefore,   
\begin{align*}
&\EE_{\bar \Dcal}[\Phi(\bar \Dcal)] 
\\ & =\EE_{\bar \Dcal}\left[\sup_{\theta \in \Theta} \EE_{(\bar x',\bar y')}[\ell(f(\bar x';\theta),\bar y')]-L(\theta;  \mathcal{\bar D})+L(\theta;\mathcal{\bar D})-L_{q}(\theta;\mathcal{\bar D}) \right] 
\\ & \le \EE_{\bar \Dcal}\left[\sup_{\theta \in \Theta} \EE_{(\bar x',\bar y')}[\ell(f(\bar x';\theta),\bar y')]-L(\theta;\mathcal{\bar D})\right]  - \Qcal_{n}(\Theta;s,q) 
 \\ & \le\EE_{\bar \Dcal, \bar \Dcal'}\left[\sup_{\theta \in \Theta} \frac{1}{n}\sum_{i=1}^n (\ell(f(\bar x_{i}';\theta), \bar y_{i}')-\ell(f(\bar x_i;\theta),\bar  y_{i}))\right]  - \Qcal_{n}(\Theta;s,q)  \\ & \le\EE_{\xi, \bar \Dcal, \bar \Dcal'}\left[\sup_{\theta \in \Theta} \frac{1}{n}\sum_{i=1}^n  \xi_i(\ell(f(\bar x_{i}';\theta), \bar y_{i}')-\ell(f(\bar x_i;\theta),\bar y_{i}))\right]  - \Qcal_{n}(\Theta;s,q) 
  \\ & \le 2\Rfra_{n}(\Theta)   - \Qcal_{n}(\Theta;s,q). 
\end{align*}
where the third line and the last line follow the subadditivity of supremum, the forth line follows the Jensen's inequality and the convexity of  the 
supremum, the fifth line follows that for each $\xi_i \in \{-1,+1\}$, the distribution of each term $\xi_i (\ell(f(\bar x_{i}';\theta),\bar y_{i}')-\ell(f(\bar x_i;\theta),\bar y_{i}))$ is the  distribution of  $(\ell(f(\bar x_{i}';\theta),\bar y_{i}')-\ell(f(\bar x_i;\theta),\bar y_{i}))$  since $\bar \Dcal$ and $\bar \Dcal'$ are drawn iid with the same distribution. Therefore, for any $\delta>0$, with probability at least $1-\delta$,

$$
\Phi(\Dcal) \le2\Rfra_{n}(\Theta)   - \Qcal_{n}(\Theta;s,q)+ \frac{Ms}{q} \sqrt{\frac{\ln(1/\delta)}{2n}}.
$$ 

\qed

\section{Additional discussion} \label{sec:app:discuss}
 
The subset $\Theta$ in Theorem \ref{thm:ge_1} characterizes the hypothesis space that is $\{x \mapsto f(x;\theta): \theta \in \Theta\}$. An important subtlety here is that  given a parameterized model $f$, one can apply Theorem \ref{thm:ge_1} to a subset $\Theta$ that depends on an algorithm and a distribution (but not directly on a dataset) such as $\Theta = \{\theta \in \RR^{d_y}: (\exists \Dcal \in A)[\text{$\theta$ is the possible output of  ordered SGD given $(f,\Dcal) $}]\}$ where $A$ is a fixed set of the training datasets such that  $\mathcal{D} \in A$ with high probability. Thus, even for the exact same model $f$ and problem setting, Theorem \ref{thm:ge_1}  might provide non-vacuous bounds for some choices of $\Theta$ but not for other choices  of $\Theta$. 

Moreover, we can easily obtain data-dependent bounds from Theorem \ref{thm:ge_1} by repeatedly applying Theorem \ref{thm:ge_1} to several subsets $\Theta$ and taking an union bound. For example, given a sequence $(\Theta_k)_{k\in \NN^{+}}$,   by applying Theorem \ref{thm:ge_1} to each $\Theta_k$ with $\delta=\delta' \frac{6}{\pi^2 k^2}$ (for each $k$) and by taking a union bound over all $k \in \NN^{+}$,  the following statement holds:  for any $\delta'>0$, with probability at least $1-\delta'$ over an iid draw of $n$  examples $\mathcal{D}=((x_i,y_i))_{i=1}^n$, we have that for all  $k \in \NN^{+}$ and $\theta \in \Theta_{k}$, 
\begin{align*}
\EE_{(x,y)}[\ell(f(x;\theta),y)] 
\le L_{q}(\theta;\Dcal) + 2\Rfra_{n}(\Theta_{k}) + \frac{Ms}{q} \sqrt{\frac{\ln(\pi^2 k^2/6\delta')}{2n}} - \Qcal_{n}(\Theta_{k};s,q).
\end{align*}
For example, let us choose $\Theta_k = \{\theta \in \RR^{d_y}: \|\theta\| \le c_{k}\}$ with some constants $c_{1}<c_{2}<\cdots$. Then, when  we obtain a  $\hat \theta_q$ after training based on a particular training dataset  $\Dcal$ such that $c_{\bar k-1}<\|\hat \theta_q\| \le c_{\bar k}$ for some $\bar k$, we can conclude the following:  with probability at least $1-\delta'$, $\EE_{(x,y)}[\ell(f(x;\theta),y)] \le L_{q}(\hat \theta_q;\Dcal) + 2\Rfra_{n}(\Theta_{\bar k}) + \frac{Ms}{q} \sqrt{\frac{\ln(\pi k^2/6\delta')}{2n}} - \Qcal_{n}(\Theta_{\bar k};s,q)$. This  is data-dependent in the sense that $\Theta_{\bar k}$ is selected in the data-dependent manner from $(\Theta_k)_{k\in \NN^{+}}$. This is in contrast to the fact that  as logically indicated in the theorem statement, one cannot directly apply Theorem \ref{thm:ge_1} to a single subset $\Theta$  that directly depends on training dataset; e.g., one \textit{cannot} apply Theorem \ref{thm:ge_1} to a singleton set  $\hat \Theta(\mathcal{D})=\{\hat \theta(\Dcal)\}$ where $\hat \theta(\Dcal)$ is the output of training given $\Dcal$.

\section{Additional experimental results and details}  \label{sec:app:exp}
\subsection{Additional results}

{\bf Wall-clock time. }Table \ref{tbl:wall_clock_time} summarises the wall-clock time values (in seconds) of mini-batch SGD and   ordered SGD. The wall-clock time  was computed with identical, independent, and freed GPUs for fair comparison. The wall-clock time
measures the time of the whole computations, including the extra computation of finding a set $Q$ of top-$q$ samples in $S$ in term of loss values. As it can be seen, the extra computation of finding a set $Q$ of top-$q$ samples is generally negligible. Furthermore, for larger scale problems, ordered SGD tends to be faster per epoch because of the computational saving of not using the full mini-batch for the backpropagation computation.    

\vspace{5pt}
{\bf Effect of different learning rates and mini-batch sizes.} Figures \ref{fig:new1} and  \ref{fig:new2} show the results with different learning rates and mini-batch sizes. Both use the same setting as that for CIFAR-10 with no data augmentation in  others results shown in Table \ref{tbl:test_error} and Figure \ref{fig:plots_sgd}. Figures \ref{fig:new1} and  \ref{fig:new2} consistently show improvement of ordered SGD over mini-batch SGD for all learning rates and mini-batch sizes. 

\begin{figure*}[h!]
\center
\includegraphics[width=\textwidth]{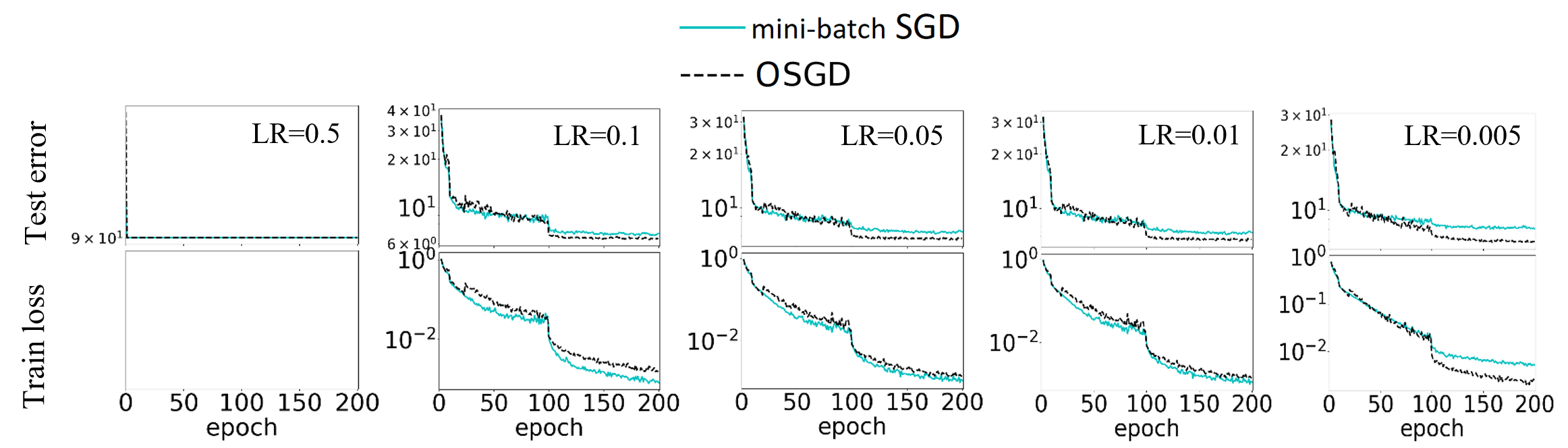}
\caption{Test error and training loss (in log scales) versus the number of epoch with  CIFAR-10 and no data augmentation by using different learning rates (LRs). The plotted values indicate the mean values over 10 random trials. The training loss values of LR=0.5 were `nan' for both methods.}
\label{fig:new1}
\end{figure*}

\begin{figure*}[h!]
\center
\includegraphics[width=0.8\textwidth]{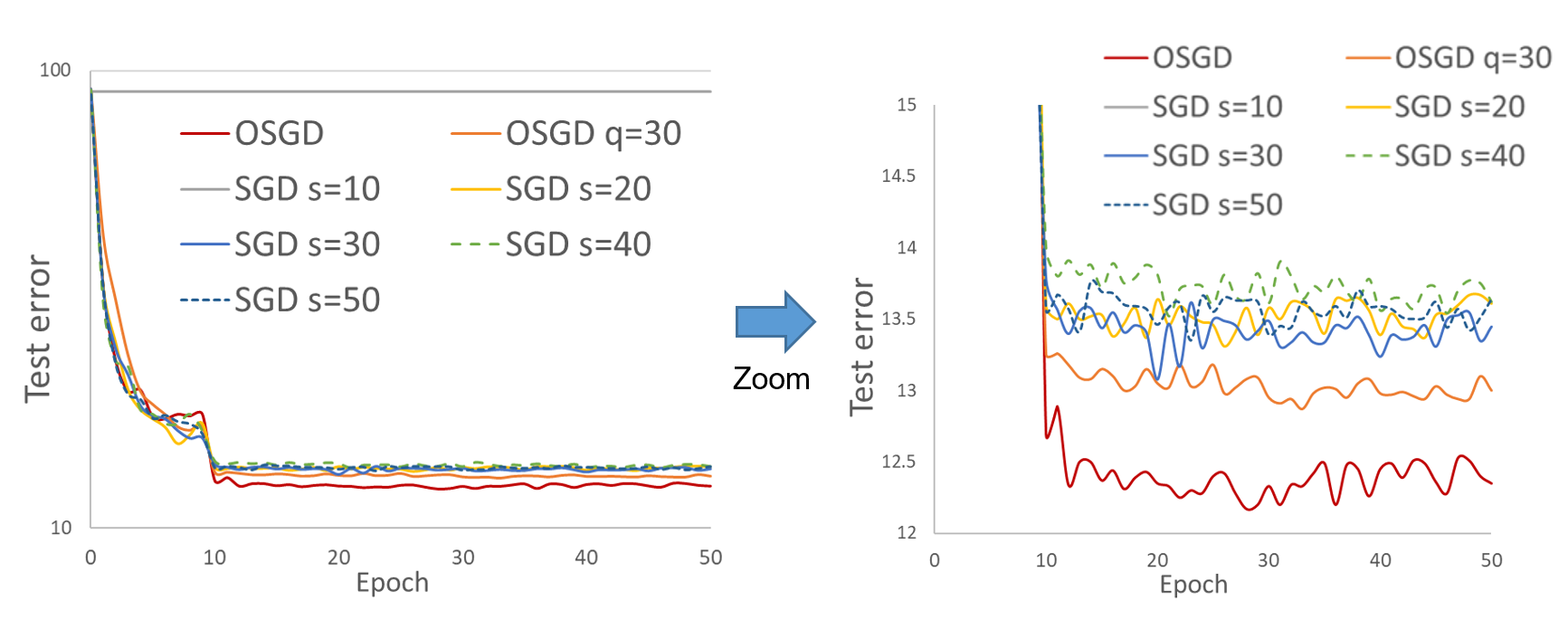}
\caption{Test error versus the number of epoch with  CIFAR-10 and no data augmentation by using different mini-batch sizes $s$.  }
\label{fig:new2}
\end{figure*}

\vspace{5pt}
{\bf Behaviors with different datasets. } Figure \ref{fig:plots_sgd_all} shows the behaviors of mini-batch SGD vs ordered SGD. As it can be seen, ordered SGD generally improved mini-batch SGD in terms of test errors. With data argumentation, we also tried linear logistic regression  for the Semeion dataset, and obtained the mean test errors of 19.11 for mini-batch SGD and 16.54 for  ordered SGD (the standard deviations were 1.48 and 1.24); i.e.,  ordered SGD improved over mini-batch SGD, but the mean test errors without data-augmentation were better  for both mini-batch SGD and ordered SGD. This is because the data augmentation made it  difficult to fit the augmented training dataset with linear models.

\vspace{5pt}
{\bf Effect of different values of $q$. }  Figure \ref{fig:plots_sgd_q} shows the behaviors of mini-batch SGD vs ordered SGD with different $q$ values. In the figure, label `ordered SGD' corresponds to  ordered SGD with the fixed adaptive rule, and other labels (e.g., `ordered SGD: $q=10$') corresponds to   ordered SGD with the fixed value of $q$ over the whole training procedure (e.g., with $q=10$). All experiments in the figure were conducted with data augmentations. PreActResNet18 was used for CIFAR-10, while LeNet was used for other datasets. As it can be seen in Figure \ref{fig:plots_sgd_q}, ordered SGD generally improved the test errors of mini-batch SGD, even with fixed $q$ values. When the value of $q$ is fixed to be small as in $q=10$,  the small $q$ value can be effective during the latter stage of training (e.g., Figure \ref{fig:plots_sgd_q} b) while the training can be inefficient during the initial stage of training (e.g., Figure \ref{fig:plots_sgd_q} c).

\vspace{5pt}
{\bf Results with ordered Adam.} Table \ref{tbl:test_error_adam} compares the testing performance of ordered Adam and (standard) Adam for different models and datasets. The table reports the mean and the standard deviation of test errors (i.e., 100 $\times$ the average of 0-1 losses on test dataset)  over $10$ random experiments with different random seeds. The procedures of ordered Adam follow those of   Adam except the additional sample strategy (line 3 - 4 of Algorithm \ref{al:qSGD}). Table \ref{tbl:test_error_adam}   shows that ordered Adam  improved Adam for all settings, except CIFAR-10 with data augmentation.
For CIFAR-10 with data augmentation, ordered SGD preformed the best among mini-batch SGD, Adam,  ordered SGD, and  ordered Adam, as it can be seen in Tables \ref{tbl:test_error} and  \ref{tbl:test_error_adam}.  

\begin{table*}[t!]
\centering \renewcommand{\arraystretch}{1.0} \fontsize{10pt}{10pt}\selectfont
\caption{Average wall-clock time (seconds) per epoch.} \label{tbl:wall_clock_time}
\begin{tabular}{lccccc}
\toprule
Data Aug & Datasets & Model & mini-batch SGD & ordered SGD & difference \\
\midrule
No & Semeion  & Logistic model & 0.15 (0.01) & 0.15 (0.01) & 0.00 \\
\midrule
No & MNIST    & Logistic model & 7.16 (0.27)  & 7.32 (0.24) & -0.16 \\
\midrule
No & Semeion  & SVM & 0.17  (0.01) & 0.17 (0.01) & 0.00 \\
\midrule
No & MNIST    & SVM & 8.60 (0.31)  & 8.72 (0.29) & -0.12 \\
\midrule
No & Semeion  & LeNet & 0.18 (0.01) & 0.18 (0.01) & 0.00\\
\midrule
No & MNIST    & LeNet & 9.00 (0.34)  & 9.12  (0.27) & -0.12 \\
\midrule
No & KMNIST    & LeNet & 9.23 (0.33) & 9.04 (0.55) & 0.19 \\
\midrule
No & Fashion-MNIST & LeNet & 8.56 (0.48) & 9.45 (0.31) & -0.90 \\
\midrule
No & CIFAR-10  & PreActResNet18 & 45.55 (0.47) & 43.72 (0.93) & 1.82 \\
\midrule
No & CIFAR-100  & PreActResNet18 & 46.83 (0.90) & 43.95 (1.03) & 2.89\\
\midrule
No & SVHN      & PreActResNet18 & 71.95 (1.40) & 66.94  (1.67) & 5.01 \\
\midrule
Yes & Semeion  & LeNet & 0.28 (0.02) & 0.28 (0.02)  & 0.00\\
\midrule
Yes & MNIST    & LeNet & 14.44 (0.54)  & 14.77 (0.41) & -0.32 \\
\midrule
Yes & KMNIST    & LeNet & 12.17 (0.33) & 11.42  (0.29) & 0.75 \\
\midrule
Yes & Fashion-MNIST    & LeNet & 12.23 (0.40) & 12.38 (0.37) & -0.14 \\
\midrule
Yes & CIFAR-10  & PreActResNet18 & 48.18 (0.58) & 46.40 (0.97) & 1.78 \\
\midrule
Yes & CIFAR-100  & PreActResNet18 & 47.37 (0.84) & 44.74 (0.91) & 2.63 \\
\midrule
Yes & SVHN      & PreActResNet18 & 72.29 (1.23) & 67.95 (1.54) & 4.34 \\
\bottomrule
\end{tabular} 
\end{table*} 

\clearpage

\begin{figure*}[p]
\vspace{-2pt}
\begin{subfigure}[b]{0.0112\textwidth}
  \includegraphics[scale=0.14]{fig/labels/y_label}
\end{subfigure}
\begin{subfigure}[b]{0.2424\textwidth}
  \includegraphics[width=\textwidth,height=0.95\columnwidth]{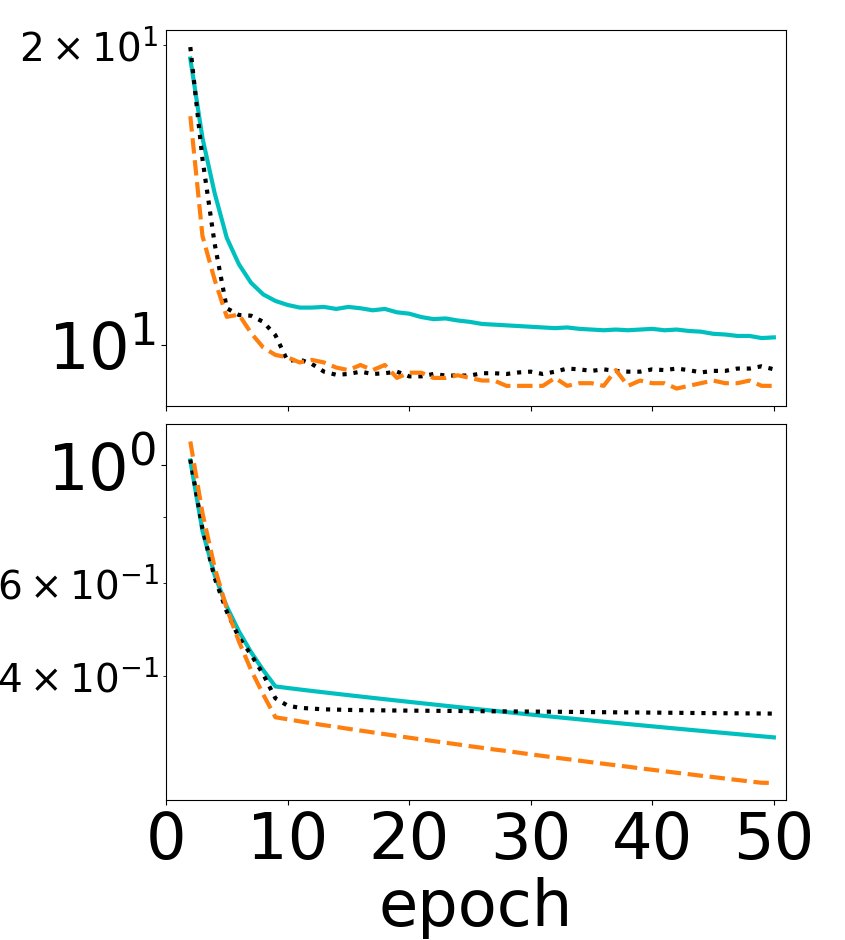}
  \caption{Semeion \&\ Logistic } 
\end{subfigure}
\begin{subfigure}[b]{0.2424\textwidth}
  \includegraphics[width=\textwidth,height=0.95\columnwidth]{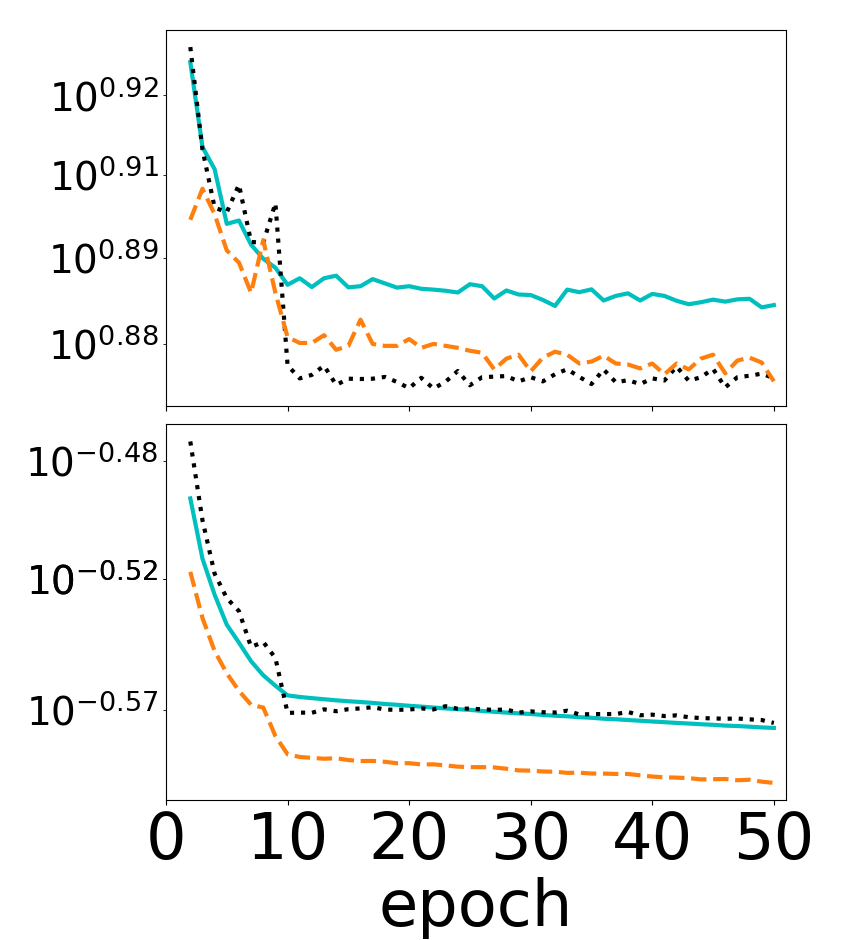}
  \caption{MNIST  \& Logistic} 
\end{subfigure}
\begin{subfigure}[b]{0.2424\textwidth}
  \includegraphics[width=\textwidth,height=0.95\columnwidth]{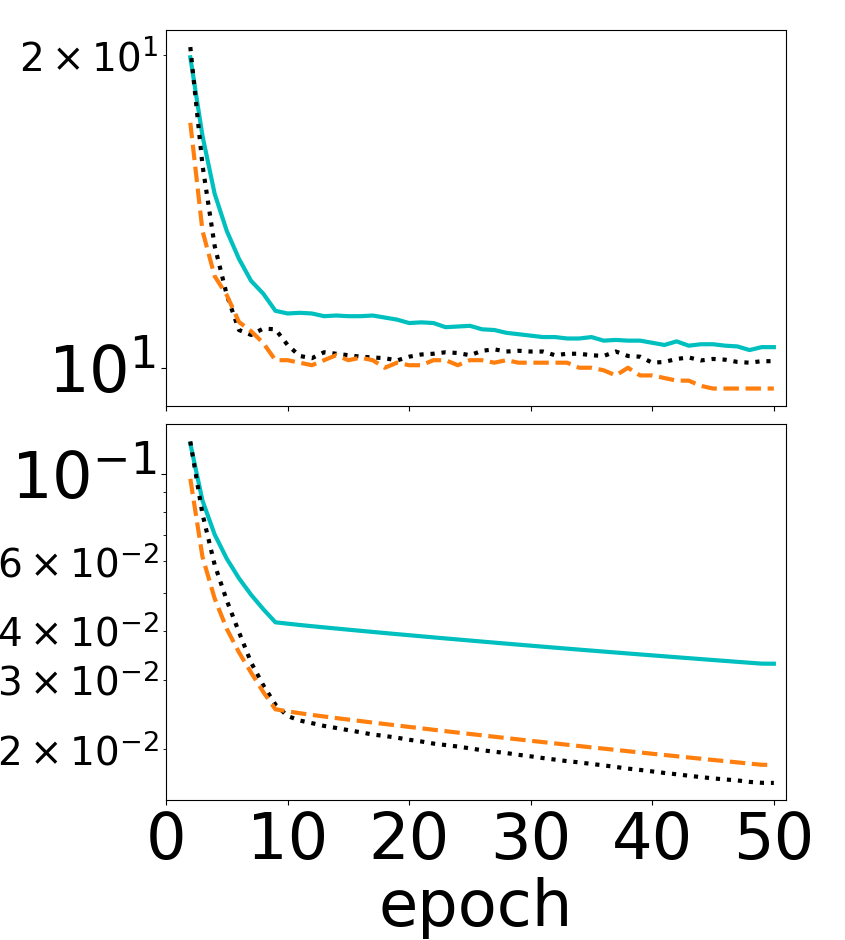}
  \caption{Semeion  \& SVM }  
\end{subfigure}
\begin{subfigure}[b]{0.2424\textwidth}
  \includegraphics[width=\textwidth,height=0.95\columnwidth]{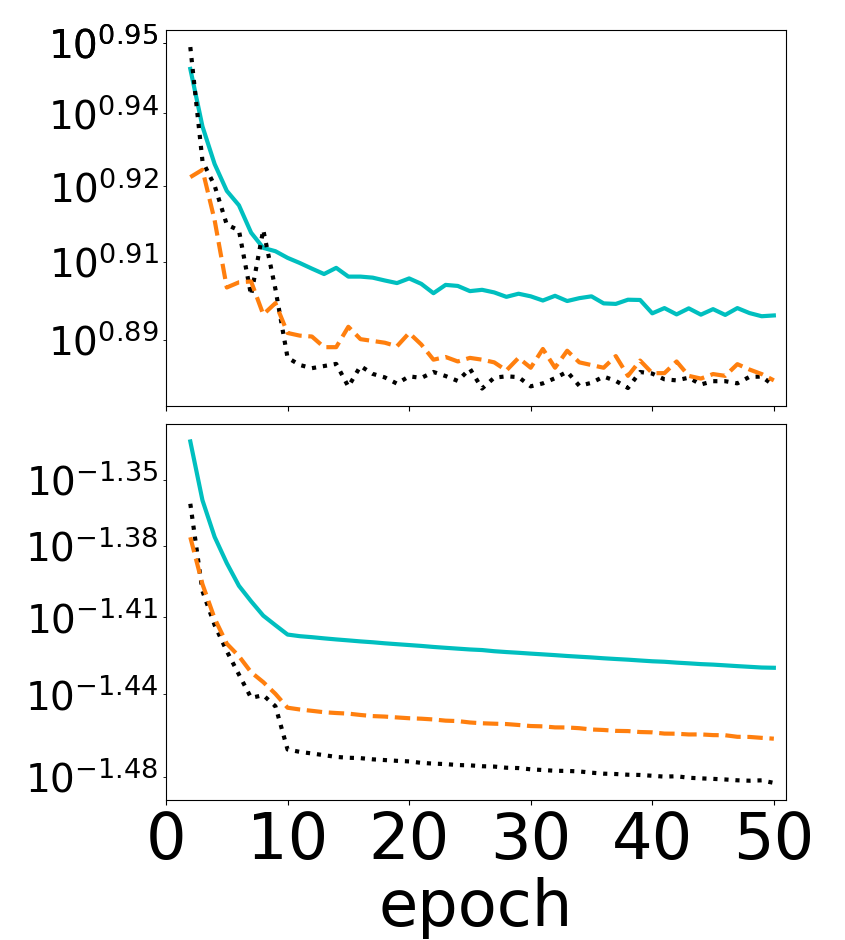}\llap{\shortstack{\includegraphics[scale=0.16]{fig/labels/legend}\\
        \rule{0ex}{1.12in}}\rule{0.14in}{0ex}}
  \caption{MNIST  \& SVM } 
\end{subfigure}
\begin{subfigure}[b]{0.0112\textwidth}
  \includegraphics[scale=0.14]{fig/labels/y_label}
\end{subfigure}
\begin{subfigure}[b]{0.2424\textwidth}
  \includegraphics[width=\textwidth,height=0.95\columnwidth]{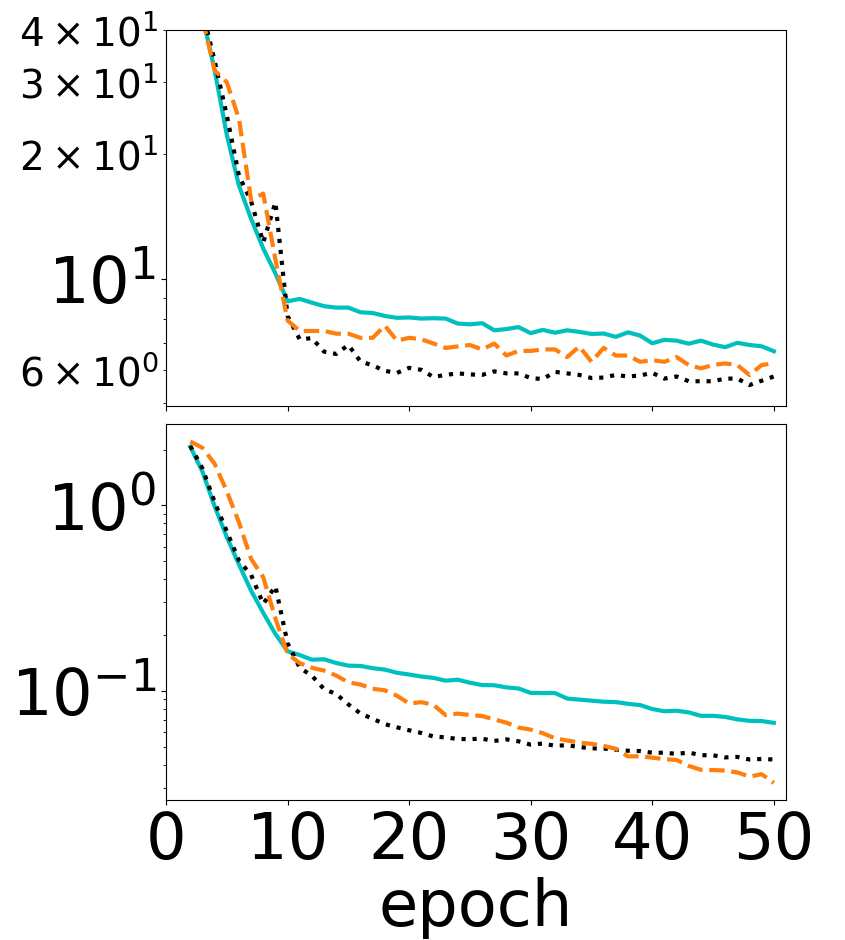}
  \caption{Semeion  \& LeNet } 
\end{subfigure}
\begin{subfigure}[b]{0.2424\textwidth}
  \includegraphics[width=\textwidth,height=0.95\columnwidth]{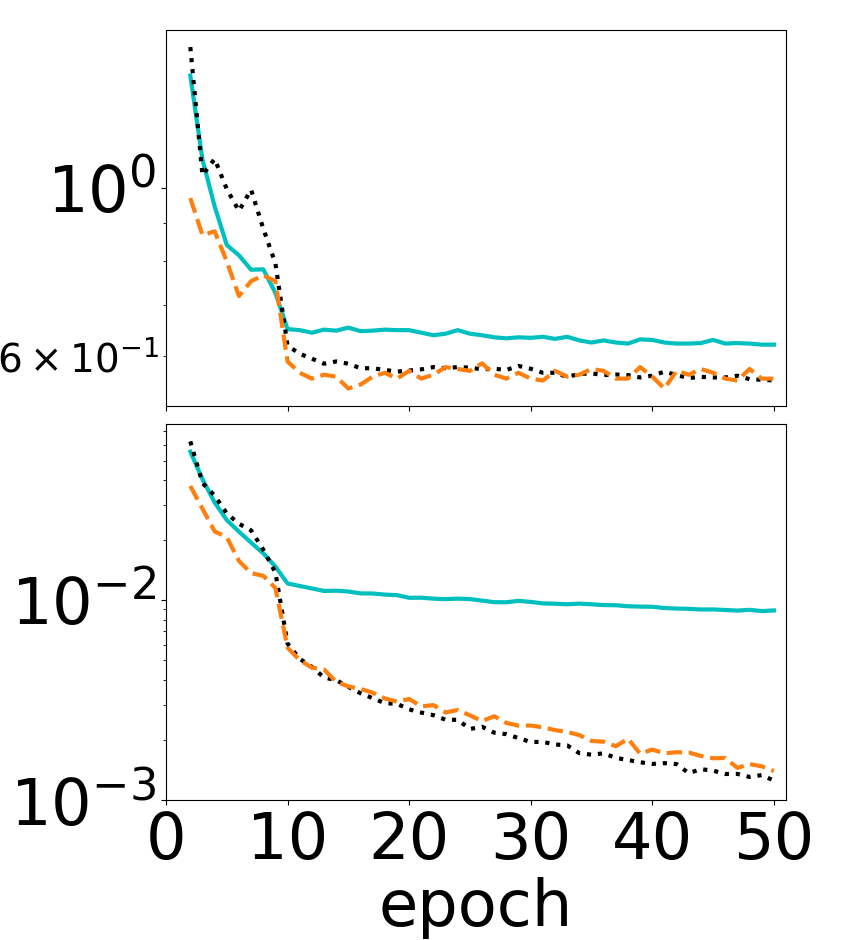}
  \caption{MNIST  \& LeNet} 
\end{subfigure}
\begin{subfigure}[b]{0.2424\textwidth}
  \includegraphics[width=\textwidth,height=0.95\columnwidth]{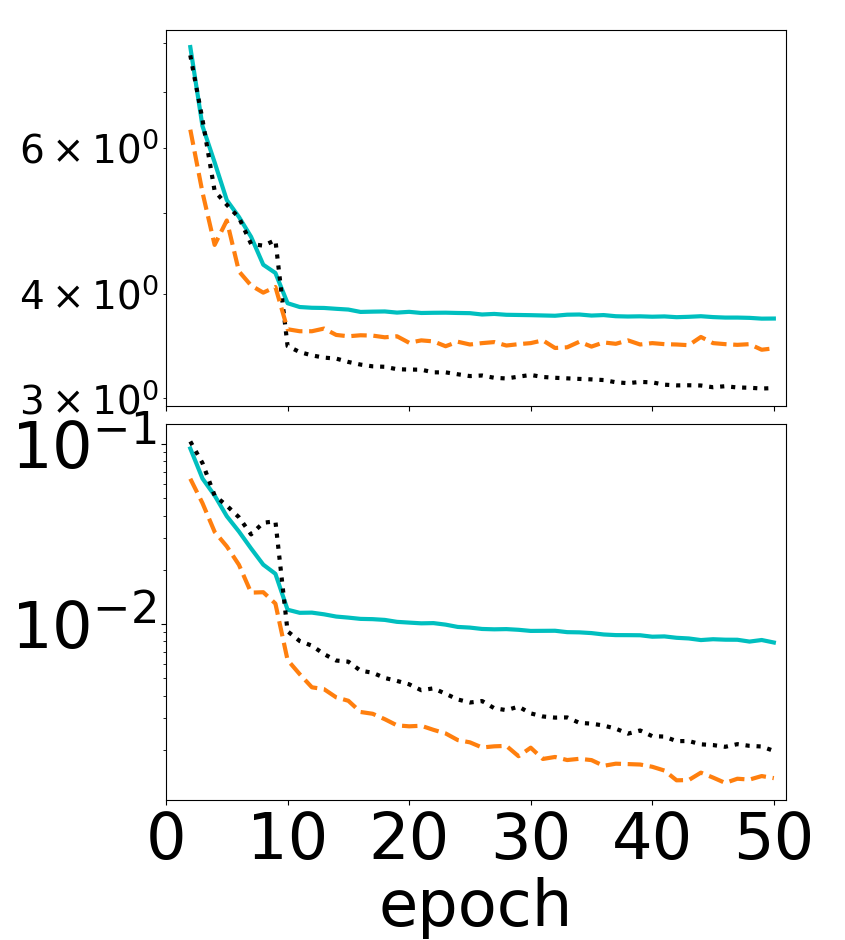}
  \caption{KMNIST} 
\end{subfigure}
\begin{subfigure}[b]{0.2424\textwidth}
  \includegraphics[width=\textwidth,height=0.95\columnwidth]{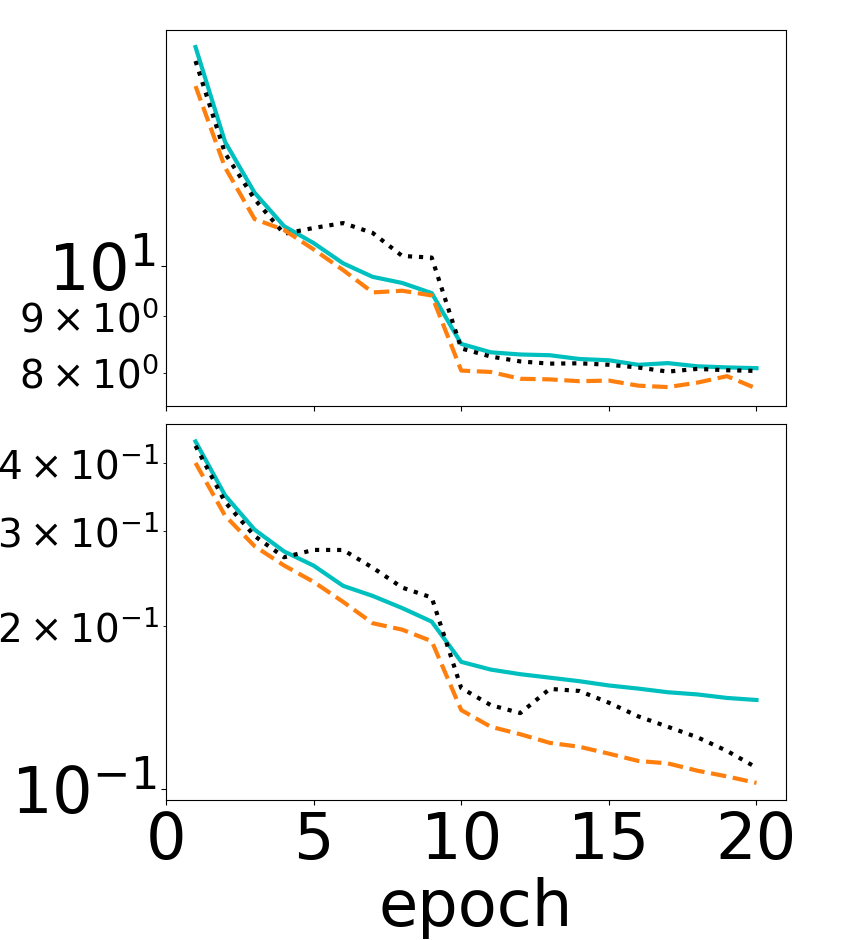}
  \caption{Fashion-MNIST} 
\end{subfigure}
\begin{subfigure}[b]{0.011\textwidth}
  \includegraphics[scale=0.14]{fig/labels/y_label}
\end{subfigure}
\begin{subfigure}[b]{0.2424\textwidth}
  \includegraphics[width=\textwidth,height=0.95\columnwidth]{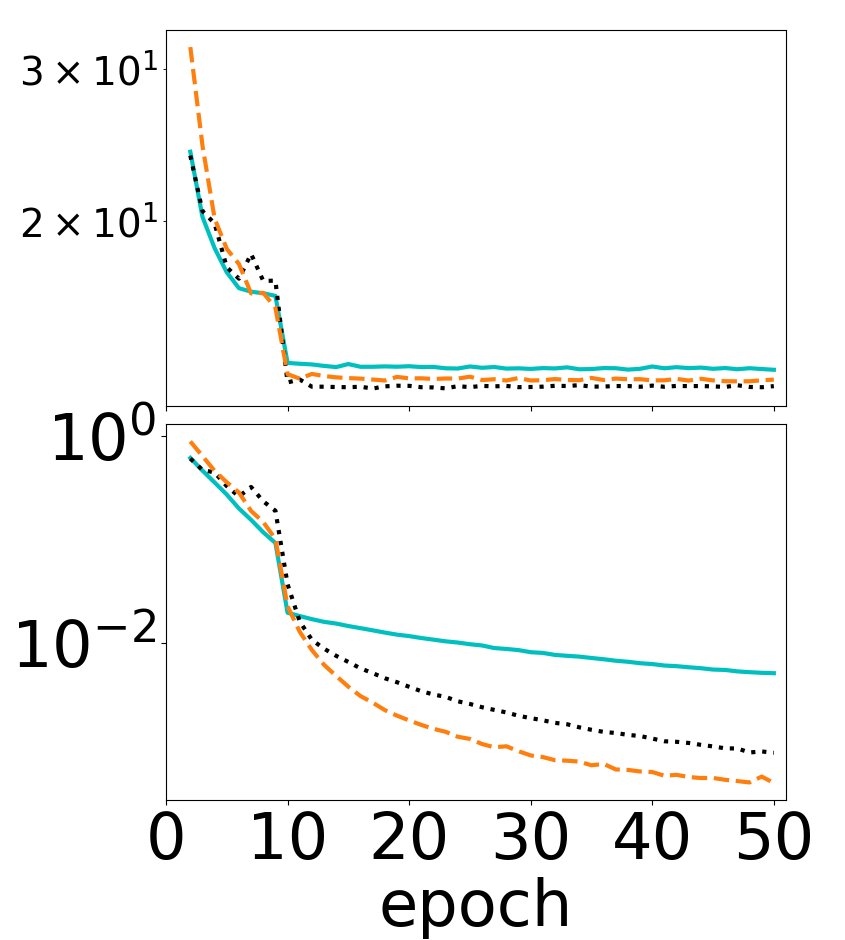}
  \caption{CIFAR-10} 
\end{subfigure}
\begin{subfigure}[b]{0.2424\textwidth}
  \includegraphics[width=\textwidth,height=0.95\columnwidth]{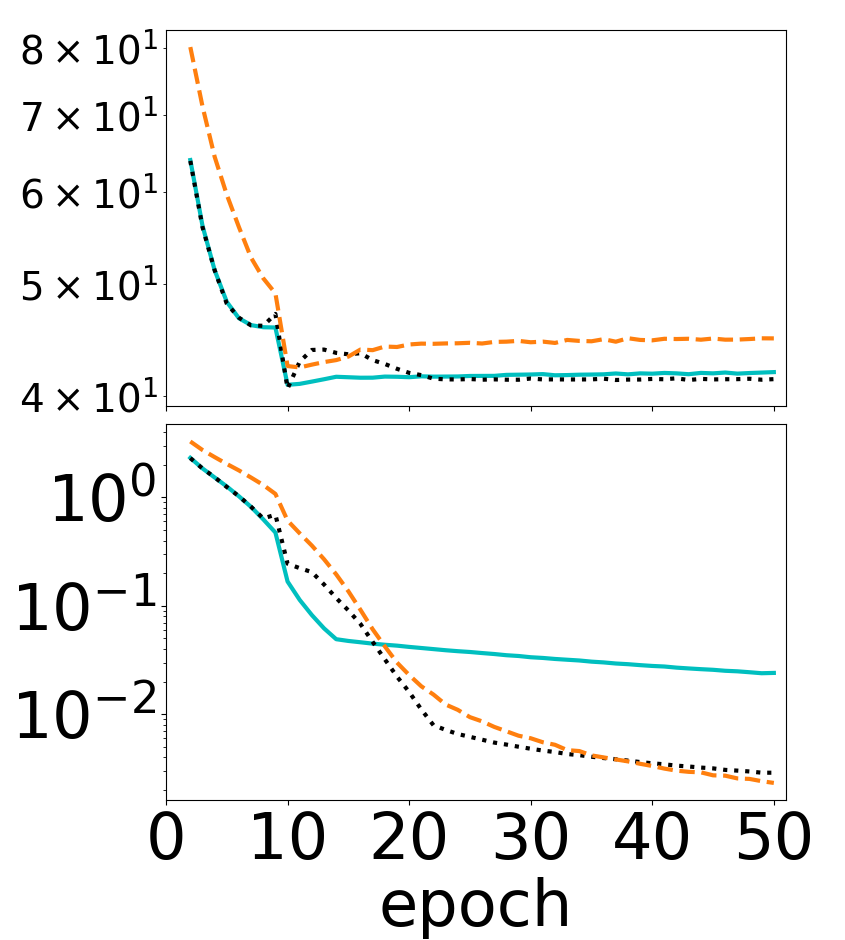}
  \caption{CIFAR-100} 
\end{subfigure}
\begin{subfigure}[b]{0.2424\textwidth}
  \includegraphics[width=\textwidth,height=0.95\columnwidth]{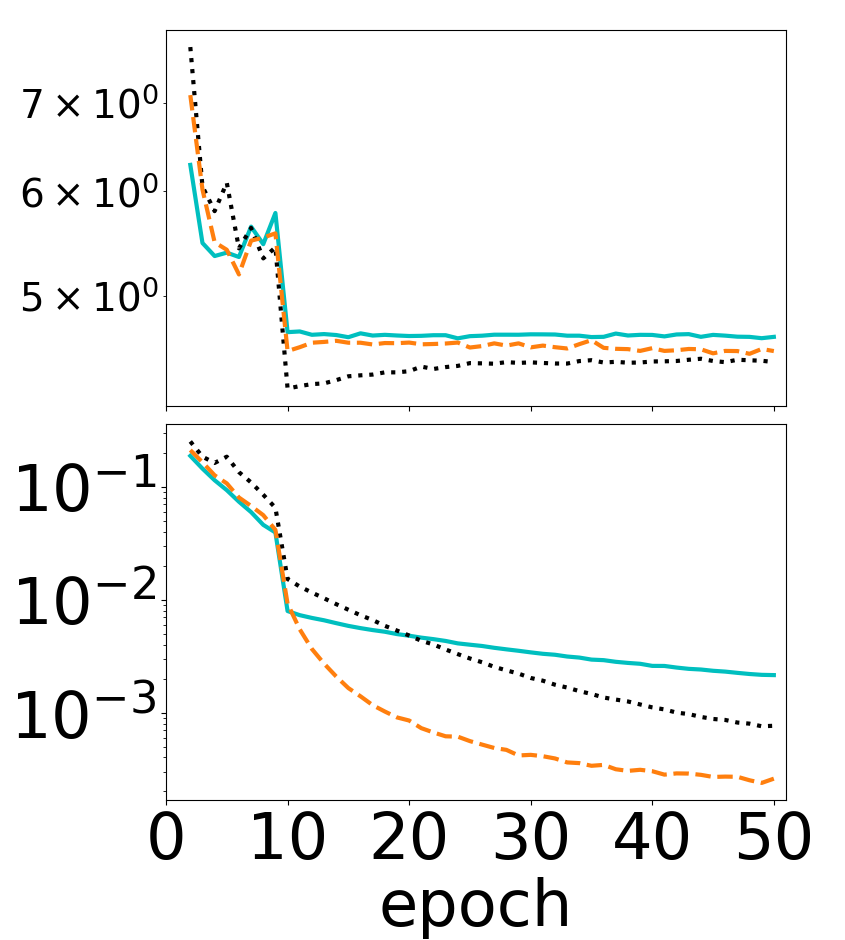}
  \caption{SVHN} 
\end{subfigure}
\begin{subfigure}[b]{0.2424\textwidth}
  \includegraphics[width=\textwidth,height=0.95\columnwidth]{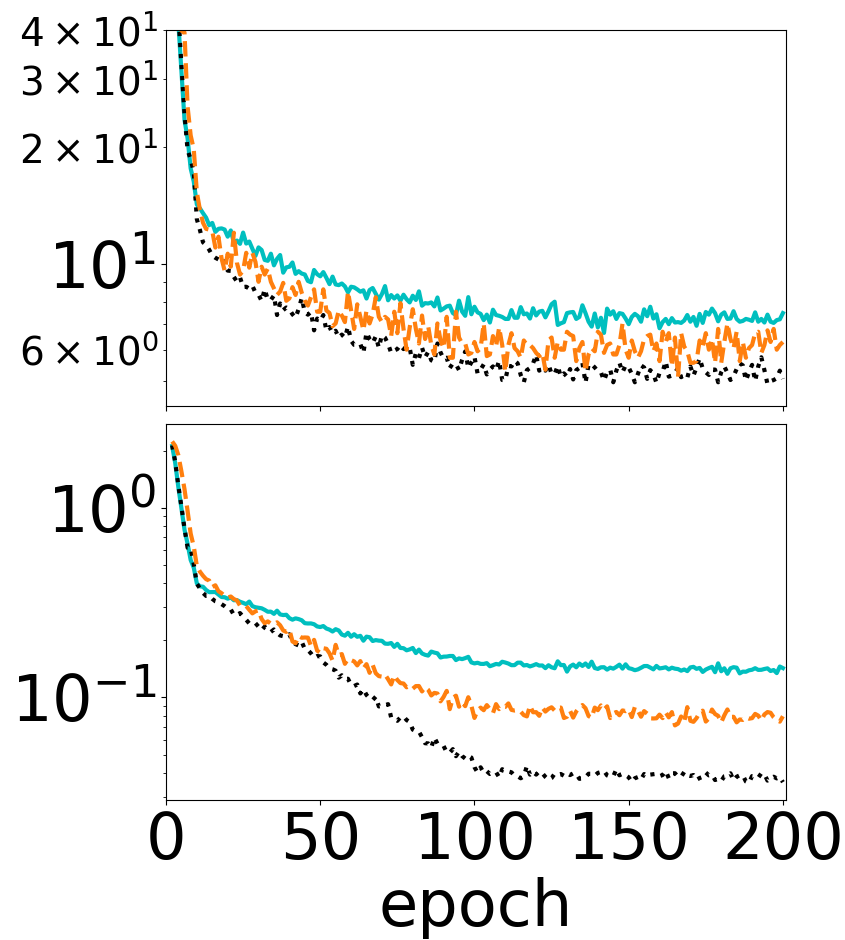}
  \caption{Semeion   \& LeNet} 
\end{subfigure}
\begin{subfigure}[b]{0.0112\textwidth}
  \includegraphics[scale=0.14]{fig/labels/y_label}
\end{subfigure}
\begin{subfigure}[b]{0.2424\textwidth}
  \includegraphics[width=\textwidth,height=0.95\columnwidth]{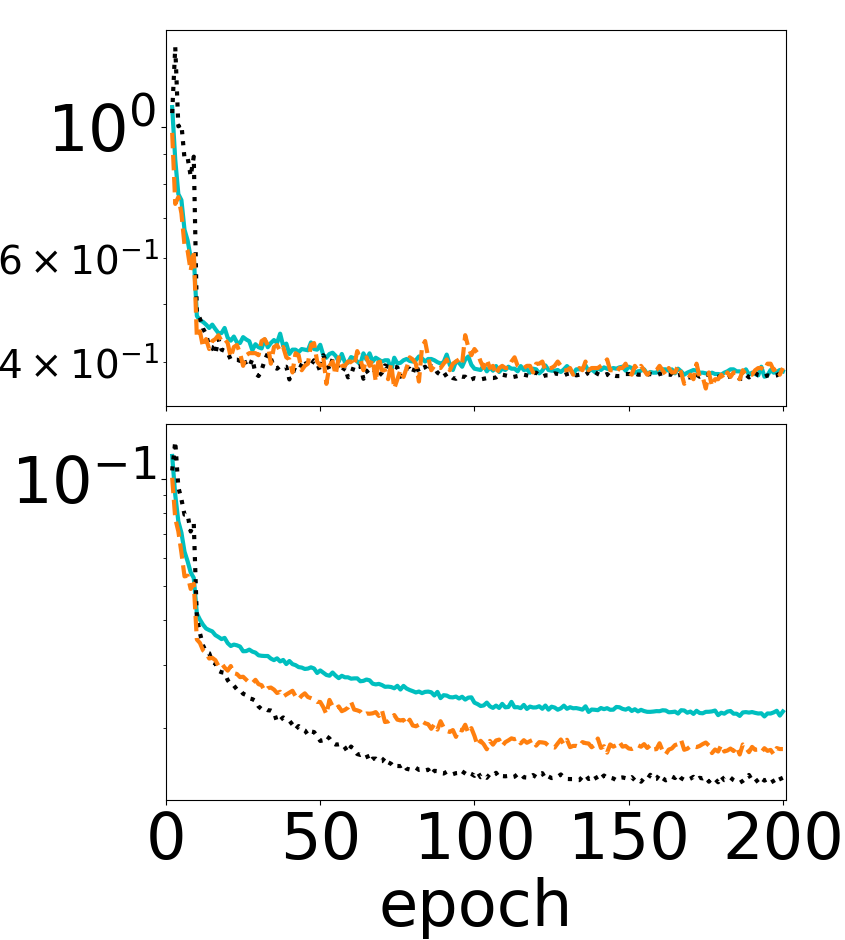}
  \caption{MNIST   \& LeNet} 
\end{subfigure}
\begin{subfigure}[b]{0.2424\textwidth}
  \includegraphics[width=\textwidth,height=0.95\columnwidth]{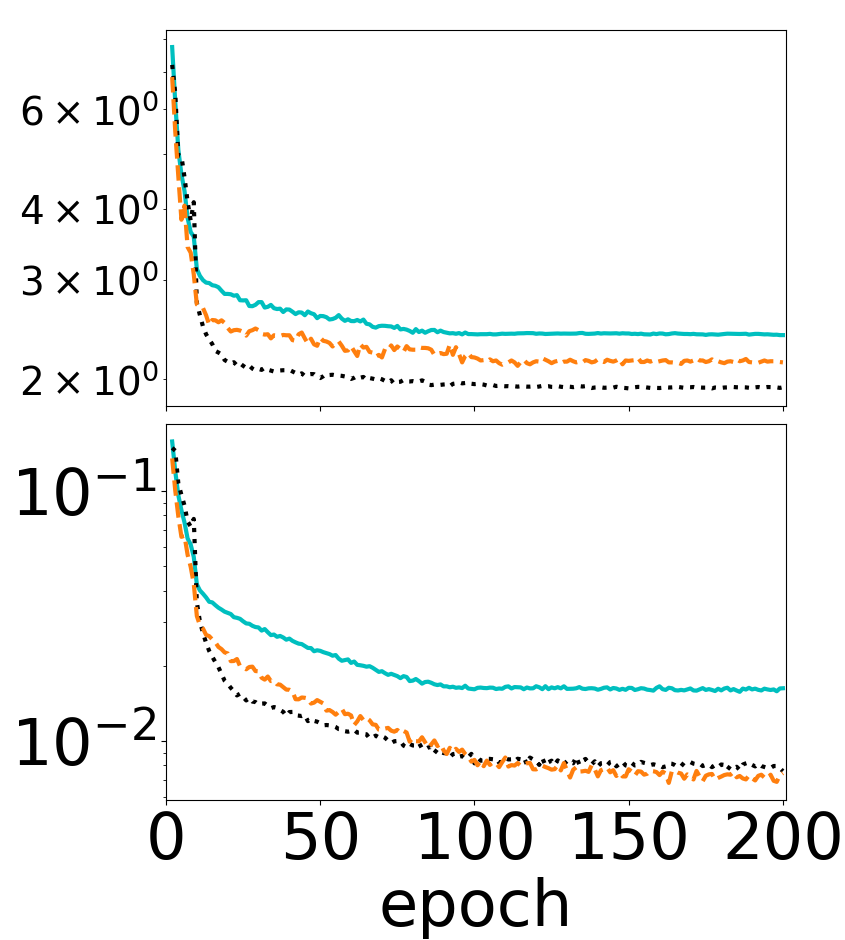}
  \caption{KMNIST } 
\end{subfigure}
\begin{subfigure}[b]{0.2424\textwidth}
  \includegraphics[width=\textwidth,height=0.95\columnwidth]{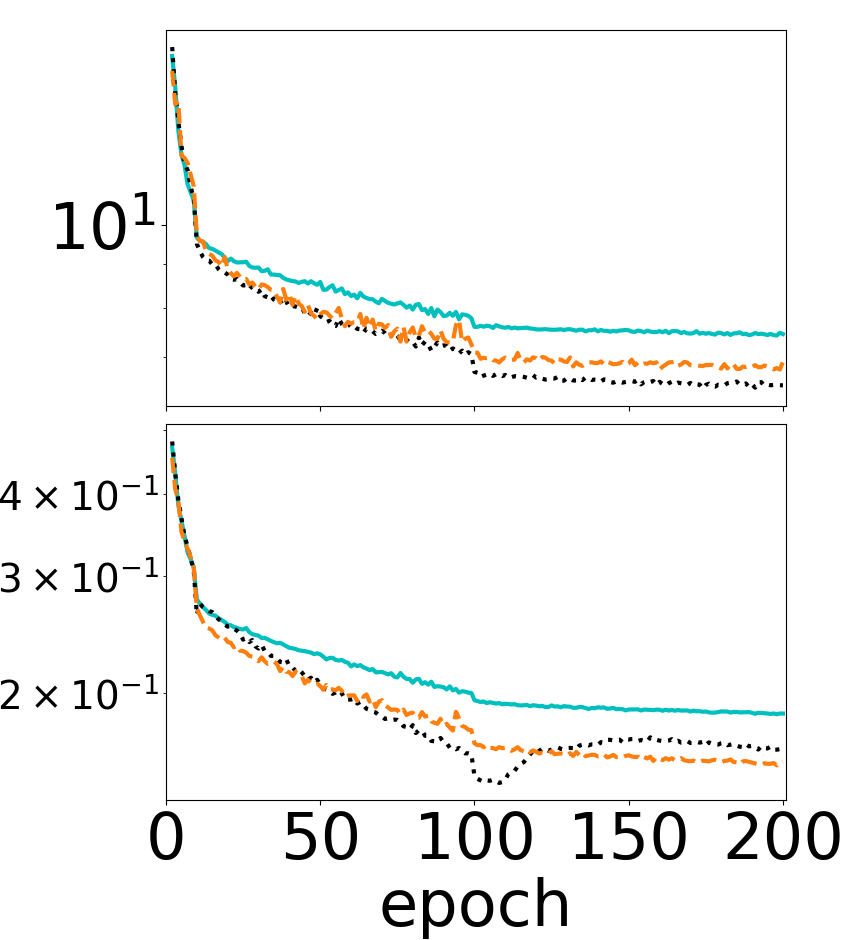}
  \caption{Fashion-MNIST} 
\end{subfigure}
\begin{subfigure}[b]{0.2424\textwidth}
  \includegraphics[width=\textwidth,height=0.95\columnwidth]{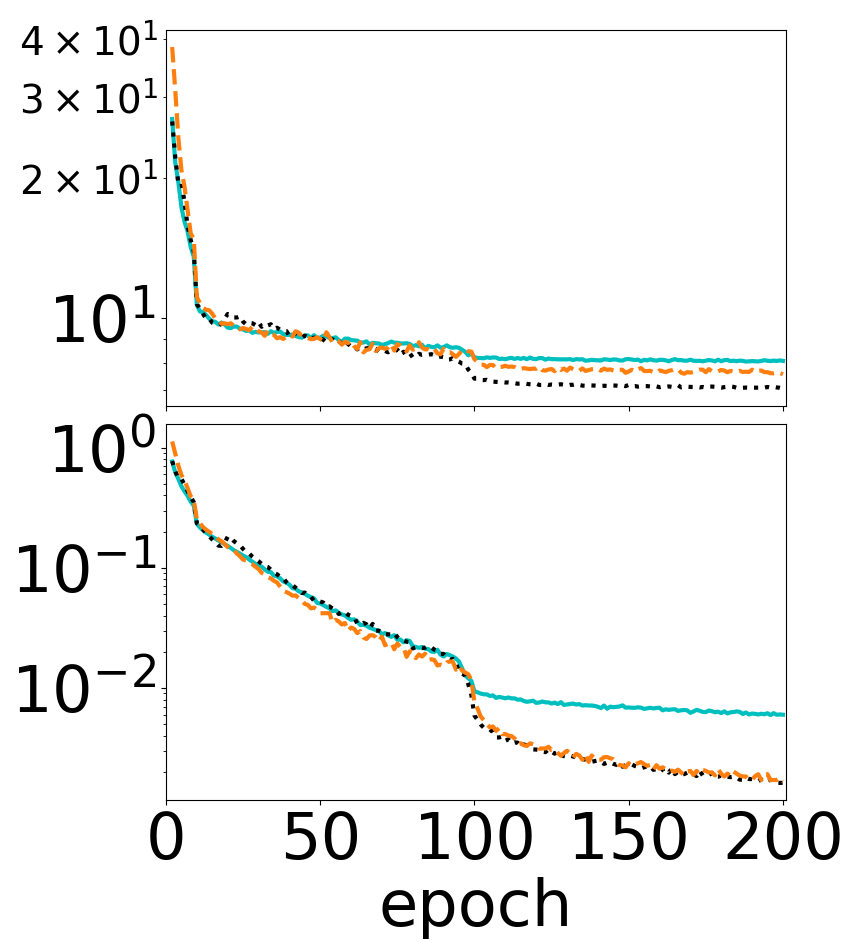}
  \caption{CIFAR-10} 
\end{subfigure}
\begin{subfigure}[b]{0.0112\textwidth}
  \includegraphics[scale=0.14]{fig/labels/y_label}
\end{subfigure}
\begin{subfigure}[b]{0.2424\textwidth}
  \includegraphics[width=\textwidth,height=0.95\columnwidth]{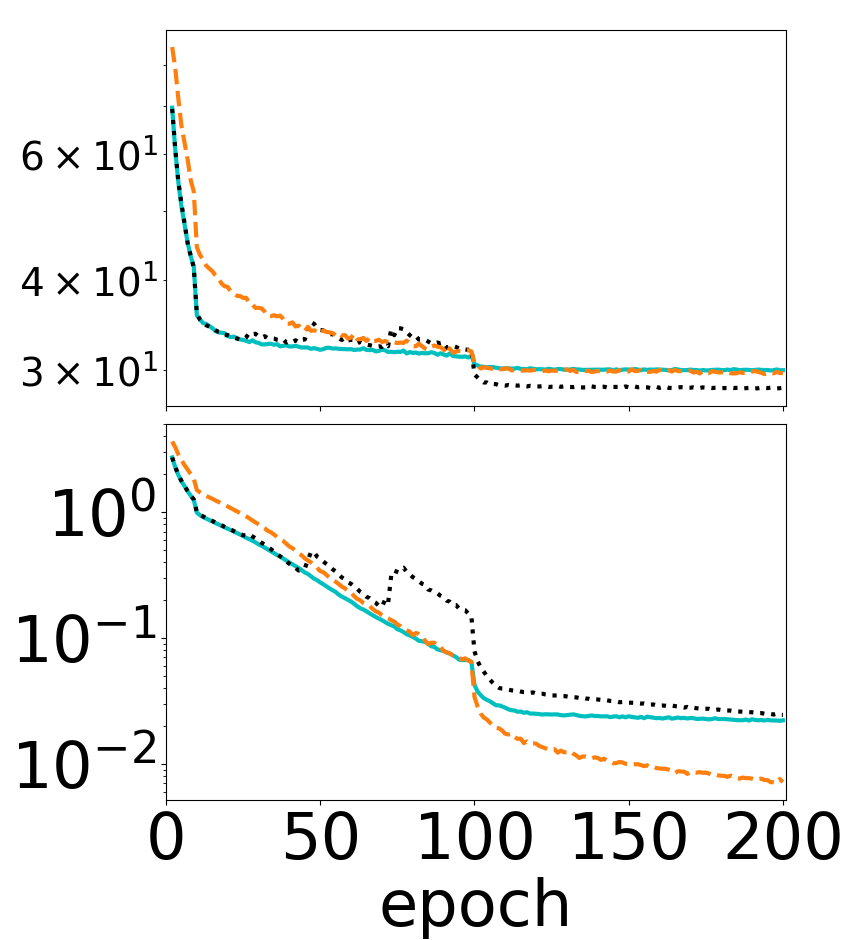}
  \caption{CIFAR-100} 
\end{subfigure}
\begin{subfigure}[b]{0.2424\textwidth}
  \includegraphics[width=\textwidth,height=0.95\columnwidth]{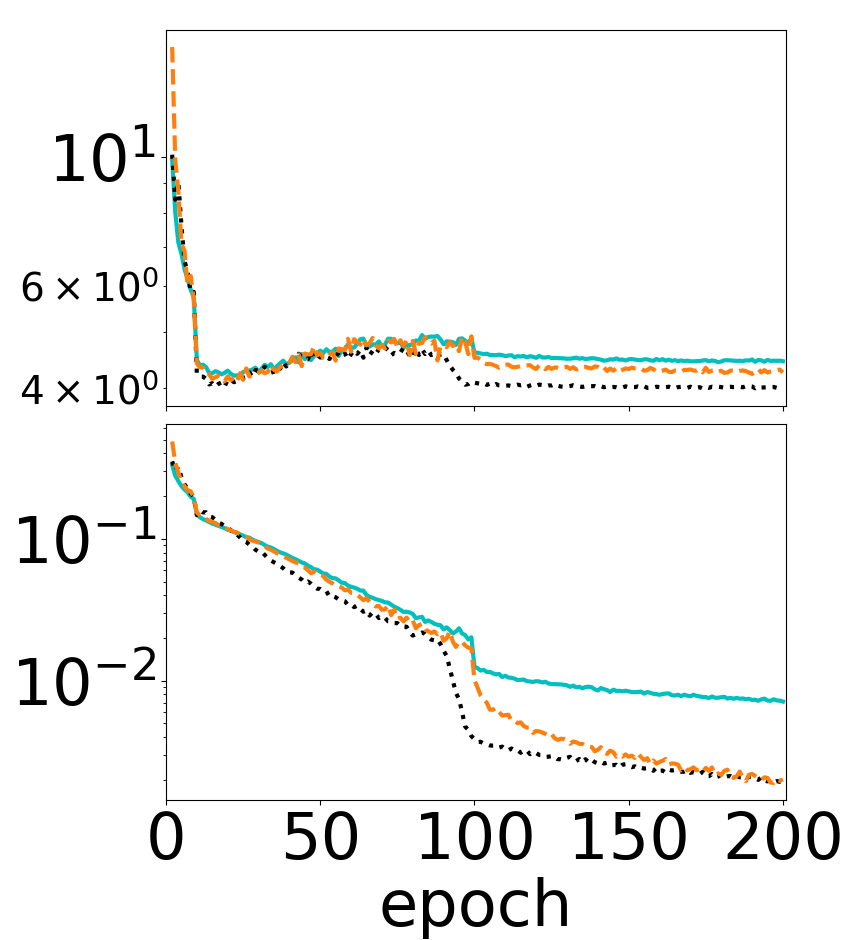}
  \caption{SVHN} 
\end{subfigure}
\caption{Test error and training loss (in log scales) versus epoch for all experiments with mini-batch SGD and ordered SGD. These are without data augmentation in subfigures (a)-(k), and with data augmentation in subfigures (l)-(r). The plotted values are the mean values over ten random trials.}
\label{fig:plots_sgd_all}
\end{figure*}

\begin{figure*}[ht!]
\begin{subfigure}[b]{0.5\textwidth}
  \includegraphics[width=\textwidth]{fig/fix_q/cifar10_1_PreActResNet18_ep_plt}
  \caption{CIFAR-10} 
\end{subfigure}
\begin{subfigure}[b]{0.5\textwidth}
  \includegraphics[width=\textwidth]{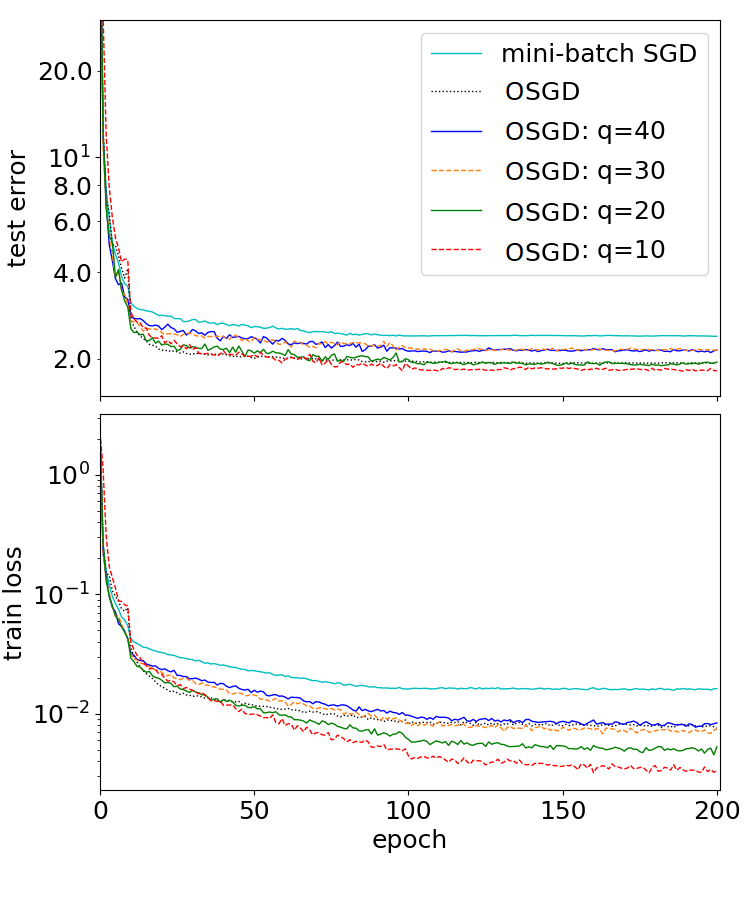}
  \caption{KMNIST} 
\end{subfigure}
\\
\begin{subfigure}[b]{0.5\textwidth}
  \includegraphics[width=\textwidth]{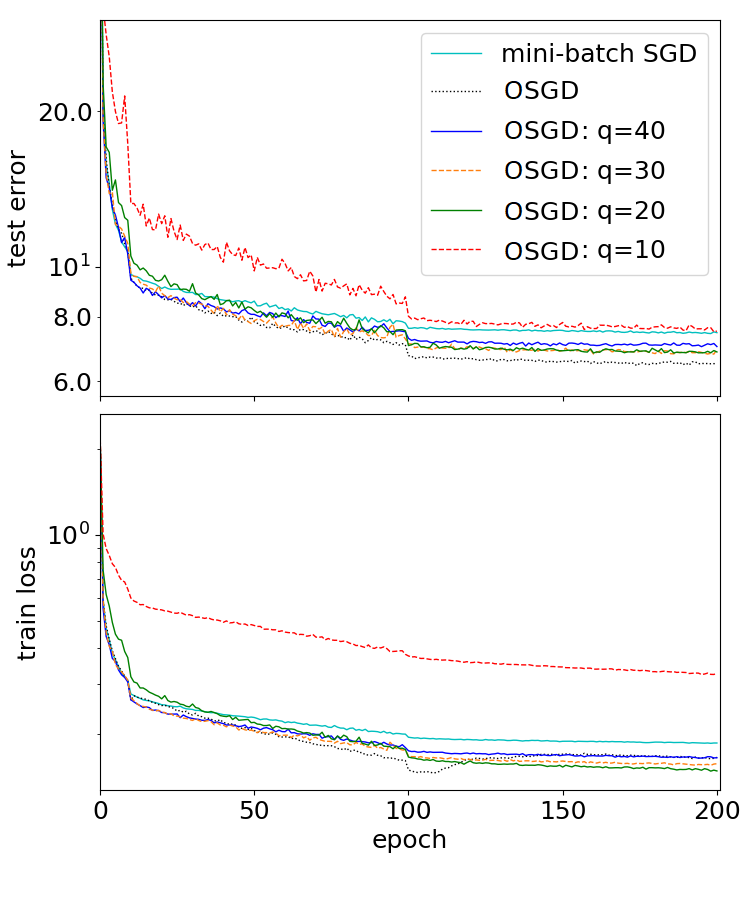}
  \caption{Fashion-MNIST} 
\end{subfigure}
\begin{subfigure}[b]{0.5\textwidth}
  \includegraphics[width=\textwidth]{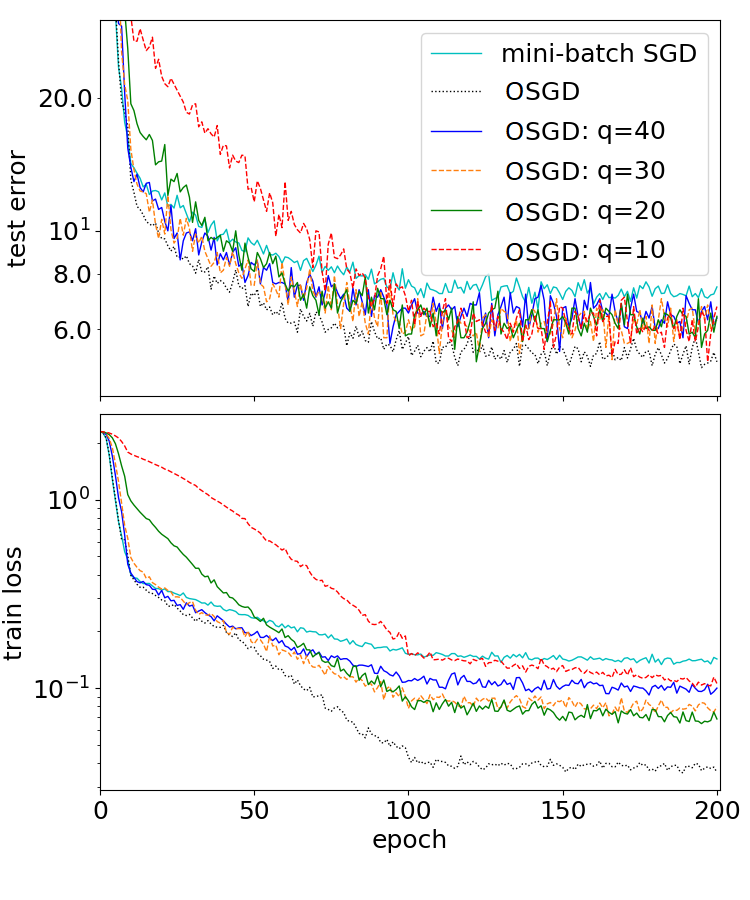}
  \caption{Semeion} 
\end{subfigure}
\caption{Effect of different values of $q$.}
\label{fig:plots_sgd_q}
\end{figure*}

\begin{table*}[p]
\centering \renewcommand{\arraystretch}{0.5} \fontsize{10pt}{10pt}\selectfont
\caption{Test errors (\%) of Adam and ordered Adam. The  last column labeled ``Improve'' shows  relative improvements (\%) from Adam to ordered Adam. In the other columns, the numbers indicate the mean test errors (and standard deviations in parentheses) over ten random trials. The first column shows `No' for  no  data augmentation, and `Yes' for data augmentation.} \label{tbl:test_error_adam} 
\begin{tabular}{lccccc}
\toprule
Data Aug & Datasets & Model &   Adam & ordered Adam & Improve\\
\midrule
No & Semeion  & Logistic model & 12.12 (0.71) & 10.37 (0.77) & 14.46 \\
\midrule
No & MNIST    & Logistic model & 7.34 (0.03)  & 7.20 (0.03)  & 1.97 \\
\midrule
No & Semeion  & SVM & 11.45  (0.90) & 10.91 (0.86) & 4.71 \\
\midrule
No & MNIST    & SVM & 7.53 (0.03)  & 7.43 (0.02)  & 1.38 \\
\midrule
No & Semeion  & LeNet & 6.21 (0.64) & 5.75 (0.42) & 7.34 \\
\midrule
No & MNIST    & LeNet & 0.70 (0.04)  & 0.63 (0.04)  & 10.07 \\
\midrule
No & KMNIST    & LeNet & 3.14 (0.13) & 3.13 (0.14) & 0.60 \\
\midrule
No & Fashion-MNIST    & LeNet & 7.79 (0.17) & 7.79 (0.21) & 0.01 \\
\midrule
No & CIFAR-10  & PreActResNet18 & 13.21 (0.42) & 12.98 (0.27) & 1.68 \\
\midrule
No & CIFAR-100  & PreActResNet18 & 45.33 (0.89) & 44.42 (0.72) & 2.01 \\
\midrule
No & SVHN      & PreActResNet18 & 4.72 (0.12) & 4.64 (0.09) & 1.52 \\
\midrule
Yes & Semeion  & LeNet & 5.80 (0.85) & 5.70 (0.60) & 1.74 \\
\midrule
Yes & MNIST    & LeNet & 0.45 (0.05)  & 0.44 (0.02)  & 3.10 \\
\midrule
Yes &KMNIST    & LeNet & 2.01 (0.08) & 1.94 (0.16) & 3.49 \\
\midrule
Yes &Fashion-MNIST    & LeNet & 6.61 (0.14) & 6.56 (0.14) & 0.82 \\
\midrule
Yes &CIFAR-10  & PreActResNet18 & 7.92 (0.28) & 8.03 (0.13) & -1.39 \\
\midrule
Yes &CIFAR-100  & PreActResNet18 & 32.24 (0.52) & 32.03 (0.52) & 0.65 \\
\midrule
Yes &SVHN      & PreActResNet18 & 4.42 (0.12) & 4.19 (0.11) & 5.29\\
\bottomrule
\end{tabular} 
\end{table*} 

\clearpage

\subsection{Additional details}
For all experiments, mini-batch SGD and ordered SGD (as well as Adam and ordered Adam) were run with the same machine and the same PyTorch codes except a single-line modification: 
\begin{itemize}
\item 
\texttt{loss = torch.mean(loss)} for mini-batch SGD and Adam
\item 
\texttt{loss = torch.mean(torch.topk(loss, min(q, s), sorted=False, dim=0)[0])} for ordered SGD and ordered Adam.  
\end{itemize} 

{\bf For 2-D illustrations in Figure \ref{fig:2d_illustration}. }We used  the (binary) cross entropy loss, $s=100$, and 2 dimensional synthetic datasets with $n=200$ in Figures \ref{fig:2d_illustration_a}--\ref{fig:2d_illustration_b} and $n=1000$ in Figures \ref{fig:2d_illustration_c}--\ref{fig:2d_illustration_d}. The artificial neural network (ANN) used in Figures \ref{fig:2d_illustration_c} and \ref{fig:2d_illustration_d}   is a fully-connected feedforward neural network with rectified
linear units (ReLUs) and three hidden layers, where each hidden layer  contained  20 neurons  in Figures \ref{fig:2d_illustration_c} and 10 neurons in Figures \ref{fig:2d_illustration_d}.

{\bf For other numerical results.  }For mixup and random erasing, we used the same setting as in the corresponding previous papers \citep{zhong2017random,verma2019manifold}. For others, we divided the learning rate by $10$ at the beginning of 10th  epoch for all experiments (with and without data augmentation), and of 100th epoch for those with data augmentation. With $y \in \{1,\dots, d_y\}$, we used the cross entropy loss  $\ell(a,y)=- \log \frac{\exp(a_y)}{\sum_{k'}\exp(a_{k'})}$ for neural networks as well as  multinomial logistic models, and  a multiclass hinge loss $\ell(a,y)= \sum_{k\neq y} \max(0,1+a_{k}-a_{y})$ for SVMs \citep{weston1999support}.  For the variant of LeNet, we used the following  architecture with five layers (three hidden layers):  
\vspace{-5pt} 
\begin{enumerate}[leftmargin=18pt]
\item \vspace{-3pt} 
Input layer
\item \vspace{-3pt} 
Convolutional layer  with 64 $5\times 5$ filters, followed
by max pooling of size of 2 by 2 and ReLU.
\item \vspace{-3pt} 
Convolutional layer  with 64 $5\times 5$ filters, followed
by max pooling of size of 2 by 2 and ReLU.
\item \vspace{-3pt} 
Fully connected layer with 1014 output units, followed by ReLU.
\item \vspace{-3pt} 
Fully connected layer with the number of output units being equal to the number of target classes.
\end{enumerate}

\end{document}